\documentclass[10pt]{article}

\usepackage[margin=1in]{geometry}
\usepackage[noadjust]{cite}
\usepackage{amsmath,amssymb,amsfonts,amsthm}
\usepackage{algorithmic}
\usepackage{graphicx,subfig}
\usepackage{textcomp}
\usepackage{xcolor}
\usepackage{tikz,adjustbox,nicefrac}
\usepackage[toc,page]{appendix}
\usepackage[font=small,labelfont=bf]{caption}

\usetikzlibrary{calc,positioning}

\makeatletter
\def\blfootnote{\xdef\@thefnmark{}\@footnotetext}
\makeatother

\renewcommand{\Re}{\mathbb{R}}

\newcommand{\Nint}{\mathcal N_{\text{int}}}
\newcommand{\T}{\mathcal T}
\newcommand{\N}{\mathcal N}
\newcommand{\C}{\mathcal C}
\newcommand{\W}{\mathcal W}
\newcommand{\Q}{\mathcal Q}

\DeclareMathOperator*{\argmax}{arg\,max}
\DeclareMathOperator*{\argmin}{arg\,min}

\renewcommand{\emptyset}{\varnothing}

\theoremstyle{definition}

\newtheorem{theorem}{Theorem}[section]
\newtheorem{lemma}[theorem]{Lemma}
\newtheorem{fact}[theorem]{Fact}
\newtheorem{proposition}[theorem]{Proposition}
\newtheorem{corollary}[theorem]{Corollary}

\newtheorem{definition}[theorem]{Definition}
\newtheorem{problem}{Problem}

\begin{document}

\title{Information-Theoretic Abstractions for Planning in Agents with Computational Constraints\\
\blfootnote{Support for this work has been provided by ONR awards N00014-18-1-2375 and N00014-18-1-2828
and by ARL under DCIST CRA W911NF-17-2-018.}
}

\author{Daniel T. Larsson\thanks{D. Larsson is a PhD student with the Guggenheim School of Aerospace Engineering, Georgia Institute of Technology, Atlanta,
		GA, 30332-0150, USA. Email:
		{\small daniel.larsson@gatech.edu}}
		\and
		Dipankar Maity\thanks{D. Maity is an Assistant Professor in the Department of Electrical and Computer Engineering at the University of North Carolina, Charlotte,
NC, 28223, USA. Email:
		{\small dmaity@uncc.edu}}
		\and
		Panagiotis Tsiotras\thanks{P. Tsiotras is the Andrew and Lewis Chair Professor with the Guggenheim School of Aerospace Engineering and the Institute for Robotics and Intelligent Machines, Georgia Institute of Technology, Atlanta,
		GA, 30332-0150, USA. Email:
		{\small tsiotras@gatech.edu}}
}
\date{}

\maketitle

\begin{abstract}
In this paper, we develop a framework for path-planning on abstractions that are not provided to the agent a priori but instead emerge as a function of the available computational resources.
We show how a path-planning problem in an environment can be systematically approximated by solving a sequence of easier to solve problems on abstractions of the original space.
The properties of the problem are analyzed, and a number of theoretical results are presented and discussed.
A numerical example is presented to show the utility of the approach and to corroborate the theoretical findings.  
We conclude by providing a discussion detailing the connections of the proposed approach to anytime algorithms and bounded rationality.
\end{abstract}

\section{INTRODUCTION}

Path and motion planning for autonomous systems has long been an area of research within the robotics and artificial intelligence communities.
This has led to the development of a number of frameworks which formulate planning tasks in terms of mathematical optimization problems, which can then be solved by utilizing approaches from optimization and optimal control \cite{Bertsekas2012,Sutton2016}. 
However, planning in complex domains can be a challenging problem, and requires the agents to spend time and computational resources in order to find solutions, giving rise to an intrinsic need for agents to balance computational complexity with optimality of the resulting plan \cite{Dean1988,Kambhampati1986,Tsiotras2012,Behnke2004,Lu2012}.

Consequently, a number of approaches within the path-planning community have been developed that aim to explicitly capture the interplay between complexity and optimality.
For example, in \cite{Tsiotras2007,Tsiotras2012,Cowlagi2008,Cowlagi2010,Cowlagi2011,Cowlagi2012,Cowlagi2012a}, the authors utilize wavelets to obtain multi-resolution representations of two-dimensional environments for path-planning.
The use of abstractions, or aggregations, of the planning space to form multi-resolution environment depictions allows these works to leverage the computational benefits of executing graph-search algorithms, such as A\(^*\), on reduced graphs of the environment that contain fewer vertices as compared to the original, full-resolution, representation.

In a similar spirit, other works, such as \cite{Kambhampati1986,Hauer2015,Hauer2016}, consider abstractions for planning, but instead utilize hierarchical representations of the world in the form of multi-resolution quadtrees and octrees.
Interestingly, the use of tree structures enables these works to incorporate environment uncertainty \cite{Kraetzschmar2004}.
With this added flexibility, these approaches can be used in an on-line manner, allowing autonomous agents to plan based on occupancy grid (OG) representations of the world that are dynamically updated as the agent interacts with the environment. 
To strike a balance between the complexity of the search and satisfactory performance, the aforementioned works recursively re-solve the planning problem as the agent traverses the world.

It should be noted that the interplay between complexity and optimality is not unique to the path-planning community.
Recent work related to bounded-rational decision making has illustrated a growing need to develop decision-making frameworks for agents that are resource limited \cite{Tishby2010,Ortega2011,Genewein2015,Larsson2017,Rubin2012}.
This area of research considers limitations in the traditional assumptions of artificial intelligence, and approaches problems by viewing agents as resource-limited entities that are constrained in terms of their information-processing capabilities. 
To model such agents, the authors in \cite{Genewein2015} utilize concepts from information theory, arguing that bounded-rational decision making can be modeled by considering Kullback-Leibler (KL) divergence constraints added to traditional maximum expected utility problems. 
Extensions of this work to sequential decision-making problems in stochastic domains is considered in \cite{Tishby2010,Rubin2012}, whereby Markov Decision Processes (MDPs) are utilized with information-theoretic constraints to formulate information-limited MDPs (IL-MDPs). 
These frameworks include a trade-off parameter that balances the optimality of the decision policy and the effort required to obtain it, as measured by a KL-divergence measure between the resulting posterior policy and a default prior policy.
These approaches offer one perspective of bounded-rational decision making and provide for interesting connections with information-theoretic frameworks for compression, such as rate-distortion theory \cite{Tishby2010,Genewein2015}.

In this paper, we consider complexity reduction in path-planning problems by means of graph abstractions for resource-limited agents.
Our approach combines aspects from both the planning and bounded-rational decision-making communities.
The contribution of this paper is two-fold.
Firstly, we employ an information-theoretic approach for the generation of multi-resolution abstractions that are a function of a single trade-off parameter and are not provided a priori for the purposes of path-planning and secondly, our framework couples the environment resolution to the resulting path quality.
To the best of our knowledge, there are no existing approaches that utilize information-theoretic abstractions for complexity reduction in path-planning that also guarantee the monotonic improvement of the path-cost as a function of environment resolution. 
The ability to couple path-cost with environment resolution allows us to facilitate connections between the path quality, the complexity of executing graph-search algorithms to obtain resolution-optimal paths and the information-processing capabilities of the agent as determined by the information-theoretically generated abstractions.
In summary, our framework (i) captures the cost of abstraction by defining an abstract cost-function, (ii) utilizes concepts from information theory to obtain reduced environment representations as a function of agent information-processing capabilities, and (iii) provides provable guarantees on the monotonic improvement of the path cost as a function of environment resolution.
The paper is organized as follows.
We begin in Section \ref{sec:prelims} by introducing necessary background material prior to formalizing our problem in Section \ref{sec:problemFormul}.
Then, in Section \ref{sec:ITIB_Abstractions}, we describe the information-theoretic framework for multi-resolution environment abstractions before discussing planning on abstractions in Section \ref{sec:planningOnAbstractions}.
We then present a numerical example in Section \ref{sec:numericalExample} with accompanying discussion in Section \ref{sec:discussion} and provide our concluding remarks in Section \ref{sec:conclusion}.
Proofs for the papers theoretical results can be found in the appendices.

\section{PRELIMINARIES} \label{sec:prelims}

Denote the set of real and non-negative real numbers by \(\Re\), \(\Re_{+}\), respectively, and, for any positive integer \(d\), let \(\Re^d\) denote the \(d\)-dimensional Euclidean space.
Assume that the environment \(\mathcal W \subset \Re^d\) is given by a \(d\)-dimensional occupancy grid and that there exists an integer \(\ell > 0\) such that the environment is contained within a hypercube of side length \(2^\ell\). 
The environment is represented as a tree \(\mathcal T = \left( \mathcal N(\mathcal T), \mathcal E(\mathcal T) \right)\), where the edge set \(\mathcal E (\mathcal T)\) describes the relationship between the nodes in \(\mathcal N(\mathcal T)\).
In this paper, we restrict our attention to the case where the tree representation is that of a quadtree, however the contributions of this paper are valid for any tree structure.
We let \(\mathcal T^{\mathcal Q}\) be the space of all feasible  quadtree representations of \(\mathcal W\), where each \(\mathcal T \in \mathcal T^{\mathcal Q}\) encodes a multi-resolution, hierarchical, representation of the world.
Take \(\mathcal T_{\mathcal W} \in \mathcal T^{\mathcal Q}\) be the quadtree corresponding to the original environment \(\mathcal W\); that is, \(\mathcal T_{\mathcal W}\) encodes the finest resolution depiction of \(\mathcal W\).

Consider any node \(n \in \mathcal N(\mathcal T_{\mathcal W})\) at depth \(k \in \left\{0,\ldots,\ell \right\}\), then \(n' \in \mathcal N(\mathcal T_{\mathcal W})\) is a child of \(n\) if the following hold:
\begin{enumerate}
		\item 
		Node \(n'\) is at depth \(k+1\) in \(\mathcal T_{\mathcal W}\),
		
		\item
		Nodes \(n\) and \(n'\) are incident to a common edge, i.e., \(\left(n,n'\right) \in \mathcal E \left( \mathcal T_{\mathcal W} \right)\).
\end{enumerate}
In the sequel, we let the set of child nodes for any \(n \in \mathcal N(\mathcal T_{\mathcal W})\) be denoted by \(\mathcal C(n)\) and \(\mathcal N_k(\mathcal T_{\mathcal W})\) to be the set of nodes at depth \(k\).
For any \(\mathcal T \in \mathcal T^{\mathcal Q}\) we take \(\mathcal N_{\text{leaf}}\left( \mathcal T \right) = \left\{n' \in \mathcal N(\mathcal T) :  \mathcal C(n') \cap \mathcal N(\mathcal T) = \emptyset \right\}\) to denote the set of leaf nodes and \(\Nint(\T) = \N(\T) \setminus \N_{\text{leaf}}(\T)\) to be the set of interior nodes of the tree \(\mathcal T\).

While useful for describing the relationship between nodes in a given tree, the aforementioned sets do not describe how the nodes in the tree \( \mathcal T \in \mathcal T^{\mathcal Q}\) are related to the spatial region described by the environment \(\mathcal W\).
In order to make these connections precise, we have the following definition.
\begin{definition}[\hspace{-0.5pt}\cite{Hauer2015}] \label{def:nodeHyperCube}
	Let \(k \in \{0,\ldots,\ell\}\) and \(n \in \mathcal N_{k}(\mathcal T_{\mathcal W})\).  
	Then the node \(n\):
	\begin{enumerate}
		\item
		Is at depth \(k\) and has an \(r\)-value given by the function \(r:\mathcal N(\T_\W) \to \{0,\ldots,\ell\}\) defined by the rule \(r (n) = \ell - k\). 
		The inverse image of the function \(r\) is the set \(r^{-1}(L) = \{n \in \N(\T_\W) : r(n) \in L\}\) for any \(L \subseteq \{0,\ldots,\ell\}\).

		\item 
		Represents a hypercube \(H(n) \subseteq \mathcal W\) with side length \(2^{r(n)}\) and volume \(2^{d r(n)}\) centered at the point \(\mathbf p(n) \in \Re^d\).
			
		\item 
		The hypercubes corresponding to the nodes that are the children of \(n\) form a partition of \(H(n)\).  
		That is,
		\begin{equation*}
			H(n) = \bigcup_{n' \in \mathcal C(n)} H(n').
		\end{equation*}
	\end{enumerate}
\end{definition}

\begin{figure}[t]
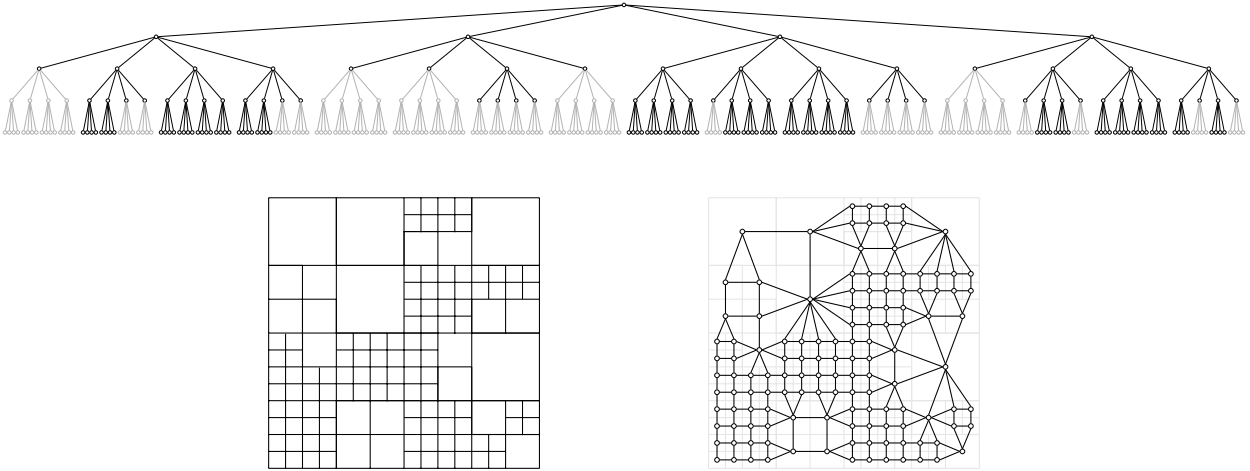

		\centering
		\begin{adjustbox}{max size={\textwidth}}

	\end{adjustbox}
		\caption{Tree representation (top) of some \(\mathcal T \in \mathcal T^{\mathcal Q}\), corresponding grid depiction (left) and associated graph (right) for a \(2^{\ell} \times 2^{\ell}\) with \(\ell = 4\) environment.  The connectivity of the graph is consistent with the definition of nodal neighbor.  The nodes in \(\mathcal T_{\mathcal W}\) that are not in \(\mathcal T\) are shown in grey.}
		\label{fig:gridAndQtreeAbstraction}
\end{figure}

In order to utilize the tree \(\mathcal T \in \mathcal T^{\mathcal Q}\) for planning, we must specify how the nodes in the tree \(\mathcal T\) are connected.
To this end, we consider the nodes \(n,\hat n \in \mathcal N\left( \mathcal T_{\mathcal W} \right)\) as \emph{nodal neighbors} if the following statements hold:
	\begin{enumerate}
		\item 
		\(\lVert \mathbf{p}(n) - \mathbf{p}(\hat n) \rVert_{\infty} = 2^{r(n) - 1} + 2^{r(\hat n)  - 1}\),
		
		\item 
		There exists a unique \(i \in \left\{1,\ldots,d \right\}\) such that \(\lvert \left[\mathbf{p}(n) - \mathbf{p}(\hat n)\right]_i\rvert = 2^{r(n) - 1} + 2^{r(\hat n)  - 1}\),
	\end{enumerate}
	where \(\left[ \mathbf{p}(n) - \mathbf{p}(\hat n) \right]_i\) denotes the \(i^{\text{th}}\) entry of the vector \(\mathbf{p}(n) - \mathbf{p}(\hat n)\) and \(|\cdot|\) denotes the absolute value.
Then, for each tree \(\mathcal T \in \mathcal T^{\mathcal Q}\) there exists an associated graph \(\mathcal G(\mathcal T) = \left(\mathcal V(\mathcal T), E(\mathcal T) \right)\), constructed from the leaf nodes of \(\mathcal T\), consisting of a set of vertices \(\mathcal V(\mathcal T)\) and edges \(E(\mathcal T)\), where the set \(E(\mathcal T)\) describes the connectivity of the vertices in \(\mathcal V (\mathcal T)\).
To describe the relation between \(\mathcal V(\mathcal T)\) and \(\mathcal N_{\text{leaf}}\left(\mathcal T\right)\), we define the mapping \(\texttt{Node}_{\mathcal G(\mathcal T)}:\mathcal V(\mathcal T) \to \mathcal N(\mathcal T_{\mathcal W})\) such that if \(n_v \triangleq \texttt{Node}_{\mathcal G(\mathcal T)}(v)\), then the vertex \(v \in \mathcal V(\mathcal T)\) corresponds to the node \(n_v \in \mathcal N(\mathcal T_{\mathcal W})\)\footnote{The mapping \(\texttt{Node}_{\mathcal G(\mathcal T)}\) has co-domain \(\mathcal N(\mathcal T_{\mathcal W})\) since the set \(\mathcal N(\mathcal T_{\mathcal W})\) contains all nodes of any tree \(\mathcal T \in \mathcal T^{\mathcal Q}\).}.
Thus, for any two vertices \(v,\hat v \in \mathcal V(\mathcal T)\), \((v,\hat v) \in E(\mathcal T)\) if and only if the nodes \(n_v,n_{\hat v} \in \mathcal N_{\text{leaf}}(\mathcal T) \subseteq \mathcal N(\mathcal T_{\mathcal W})\) are nodal neighbors.
A visual depiction of this relation is provided in Figure \ref{fig:gridAndQtreeAbstraction}. 
Note that the above process describes how a graph \(\mathcal G(\mathcal T)\) can be constructed when given any tree \(\mathcal T \in \mathcal T^{\mathcal Q}\).

Before discussing the application to path-planning, we require the formalism of a probability space and its relation to the tree \(\mathcal T_{\mathcal W}\).
The need for this stems from the information-theoretic framework utilized for compression and the technical considerations pertaining to the probibalistic encoding of obstacles by the OG.
To this end, let \(\left( \Omega, \mathcal F, \mathbb P\right)\) be a probability space with finite sample space \(\Omega\), \(\sigma\)-algebra \(\mathcal F\), and probability measure \(\mathbb P:\mathcal F \to [0,1]\).
Define random variables \(X:\Omega \to \mathcal N_{\text{leaf}}\left( \mathcal T_{\mathcal W} \right)\) and \(Y:\Omega \to \{0,1\}\), with associated distributions \(p(x) = \mathbb{P}\left( \left\{\omega \in \Omega: X(\omega) = x \right\} \right)\) and \(p(y) = \mathbb{P}\left( \left\{\omega \in \Omega: Y(\omega) = y \right\} \right)\).
The random variables \(X\) and \(Y\) can then be viewed as representing each of the unit hypercubes of \(\mathcal W\) and the total cell occupancy, respectively, where for \(y \in \Omega_Y = \{0,1\}\), we let \(y = 1\) represent the outcome of ``occupied'' and \(y = 0\) correspond to the outcome of ``empty''.
The OG representation of \(\mathcal W\) then provides us with the conditional distribution \(p(y=1 | x)\) for all \(x \in \Omega_X\), where \(p(y=1 | x)\) is the probability that the cell \(x \in \Omega_X\) is occupied.

\vspace{0.4cm}

\section{PROBLEM FORMULATION} \label{sec:problemFormul}
The problem we are interested in is defined as follows.

\begin{problem} \label{prob:frp_approx}
Given the tree \(\mathcal T_{\mathcal W}\), a scalar \(\varepsilon \in [0,1]\), constants \(\lambda_1 \in (0,1], \lambda_2 \in [0,1]\) with \(\lambda = (\lambda_1,\lambda_2)\), a start node \(s_0 \in \mathcal N_{\text{leaf}}(\mathcal T_{\mathcal W})\) and a goal node \(s_g \in \mathcal N_{\text{leaf}}(\mathcal T_{\mathcal W})\), we consider the problem of obtaining a \emph{finest-resolution path} (FRP) \(\pi = \left\{x_0,\ldots,x_K \right\} \subseteq \mathcal N_{\text{leaf}}(\mathcal T_{\mathcal W})\) where \(x_0 = s_0\), \(x_K = s_g\), each \(x \in \pi\) is distinct, and \(x_i,x_{i+1} \in \pi\) are nodal neighbors for all \(i \in \{0,\ldots,K-1\}\), so as to satisfy 
\begin{equation} \label{eq:FRPpathProblem}
\pi^* \in \argmin_{\pi \in \Pi} J^{\lambda}_{\varepsilon}(\pi),
\end{equation}
where 
\begin{equation}\label{eq:FRPpathCost}
J^{\lambda}_{\varepsilon}(\pi) = \sum_{x \in \pi} c^{\lambda}_{\varepsilon}(x),
\end{equation}
and 
\begin{equation}\label{eq:FRPcost}
c^{\lambda}_{\varepsilon}(x) = \begin{cases}
\lambda_1 + \lambda_2 p(y = 1 |x), & \text{if} ~ x \in \mathcal P_{\varepsilon},\\
M_{\varepsilon}^{\lambda}, & \text{if} ~ x \in \mathcal{N}_{\text{leaf}}(\mathcal T_{\mathcal W}) \setminus \mathcal P_{\varepsilon},
\end{cases}
\end{equation}
with \(M_{\varepsilon}^{\lambda} = 2^{d\ell}(\lambda_1 + \varepsilon \lambda_2) + \gamma\) for any \(\gamma > 0\),\footnote{Strictly speaking, \(\gamma > 0\) may be any positive number.  However, we let \(\gamma = 2\) in this paper.} \(\mathcal P_{\varepsilon} = \left\{ x \in \Omega_X: p( y= 1 | x)\leq \varepsilon \right\}\) and where \(\Pi\) denotes the set of FRPs leading from the start node \(s_0\) to the goal \(s_g\) in the tree \(\mathcal T_{\mathcal W}\).
We aim to approximate \eqref{eq:FRPpathProblem} to various degrees of fidelity by leveraging environment abstractions that can be tailored to agent resource constraints so as to reduce the computational complexity of the planning problem.
\hfill \(\triangle\) 
\end{problem}

We call an FRP \(\pi\) for which \(\pi \subseteq \mathcal P_{\varepsilon}\) an \emph{\(\varepsilon\)-feasible} FRP.
The role of \(\varepsilon\) is to define a feasible cell when the obstacle information is encoded probibalistically, and \(M^{\lambda}_{\varepsilon}\) is a factor that penalizes nodes that are considered to be obstacles so as to ensure search algorithms do not include them as part of an FRP unless no feasible paths exist. 
This is important, since nodes \(x \in \mathcal P_{\varepsilon}^{c}\) are \emph{not} removed from the search, ensuring that the right hand side of \eqref{eq:FRPpathProblem} is non-empty.
The choice of cost function \eqref{eq:FRPcost} is inspired by previous works within the robotics community that have considered planning on multi-scale abstractions \cite{Hauer2015}.
The work in this paper is distinct from existing works in that we (i) utilize a principled, information theoretic, framework to generate abstractions not provided a priori as a function of a single trade-off parameter, and (ii) provide theoretical results that couple environment resolution and path cost.

The resulting search problem on the graph \(\mathcal G(\mathcal T_{\mathcal W})\) may be computationally expensive.
However, notice that by changing the leaf nodes of the tree \(\mathcal T \in \mathcal T^{\mathcal Q}\), we alter the graph representation \(\mathcal G(\mathcal T)\) and thereby influence the number of nodes and the complexity of the resulting graph-search.
Thus, instead of solving \eqref{eq:FRPpathProblem} directly on \(\mathcal G(\mathcal T_{\mathcal W})\), we propose to approximate \eqref{eq:FRPpathProblem} by a computationally easier-to-solve problem on a graph \(\mathcal G(\mathcal T)\) for some \(\mathcal T \in \mathcal T^{\mathcal Q}\).
The challenge is then to determine how to select the tree \(\mathcal T \in \mathcal T^{\mathcal Q}\) as a function of agent resource constraints.

\section{SOLUTION APPROACH}
In this section, we review an information-theoretic framework from \cite{Larsson2020} that is utilized to generate environment abstractions.
It should be noted that while \cite{Larsson2020} presents a method for the generation of abstractions utilized in this paper, that work does not address the use of abstractions for the purposes of path-planning. 
After a brief review of the abstraction-generating process, we discuss path-planning on the resulting abstractions and show how these solutions can be used as approximations to Problem \ref{prob:frp_approx}.

\subsection{Information-Theoretic Tree Selection} \label{sec:ITIB_Abstractions}
Information theory provides a number of frameworks to construct encoders that compress arbitrary signals \cite{Cover2006,Tishby1999,Slonim2000}.
It is well known that the mutual information between a compressed representation \(Z\) of \(X\), given by 
\begin{equation}
	I(Z;X) \triangleq \sum_{z,x} p(z,x) \log\frac{p(z,x)}{p(z)p(x)},
\end{equation}
measures the amount of compression between the random variables \(X\) and \(Z\) \cite{Cover2006}.
Smaller values of the non-negative quantity \(I(Z;X)\) imply that \(Z\) is a more compressed representation of the original signal \(X\) as compared to those with larger values of \(I(Z;X)\).
However, maximizing compression via the minimization of \(I(Z;X)\) is not a well-posed problem since \(I(Z;X) = 0\) is always attainable.
Instead, one must constrain the problem by a measure that captures how good of a compressed representation \(Z\) is of \(X\).

One particular method of interest is the \emph{information bottleneck} (IB), which defines a good abstraction by the amount of information retained in the compressed representation regarding a third, relevant, random variable \cite{Tishby1999}.
More precisely, the IB method considers the problem
\begin{equation} \label{eq:IBproblemGeneral}
	p^*(z|x) = \argmax_{p(z|x)} I(Z;Y) - \frac{1}{\beta} I(Z;X),
\end{equation}
where \(X\), \(Y\), \(Z\) are random variables corresponding to the original signal, relevant variable and compressed signal, respectively, \(I(Z;Y)\) is the amount of relevant information retained in the compressed representation and \(p(z|x)\) is an encoder, mapping outcomes of \(X\) to outcomes of \(Z\).
The trade-off parameter \(\beta > 0\) balances the amount of relevant information retained in the compressed representation vs. the amount of compression of the original signal.
For discrete random variables and assuming: (i) the Markov chain \(Z \leftrightarrow X \leftrightarrow Y\), encoding the fact that \(Z\) cannot be more informative regarding \(Y\) than \(X\), and (ii) the joint distribution \(p(x,y)\) is provided, a local solution to \eqref{eq:IBproblemGeneral} can be obtained by an algorithm that likens the Blahut-Arimoto algorithm from rate-distortion theory \cite{Tishby1999}.

However, the resulting encoder \(p^*(z|x)\) is generally stochastic (i.e., \(x\) may have partial membership to multiple \(z\)), and is not guaranteed to represent a tree representation of the environment.
Consequently, direct application of traditional algorithms to solve the problem \eqref{eq:IBproblemGeneral} is not possible.
Constraining the IB problem to the space of multi-resolution tree representations of \(\mathcal W\) presents a significant challenge, and has only recently been investigated \cite{Larsson2020}.

\begin{figure}[t]
	\centering
	\begin{adjustbox}{max size={0.4\textwidth}}
		\begin{tikzpicture}[level distance=1.2cm,
		level 1/.style={sibling distance=3cm},
		level 2/.style={sibling distance=0.7cm}]
		
		\node[shape = circle, draw, line width = 1pt, minimum size = 2.5mm, inner sep = 0mm] at (0,0) {};        
		\node at (7,0) {\Large \(\mathcal T_1\) \normalsize};
		
		\node (A) at (0, -0.5) {};
		\node (B) at (0, -1.5) {};
		
		\node (A2) at (7, -0.5) {};
		\node (B2) at (7, -1.5) {};
		
		\draw [->, line width=0.5mm, black] (A) -- (B);
		\draw [->, line width=0.5mm, black] (A2) -- (B2);
		
		\node at (7,-2) {\Large \(\mathcal T_2\) \normalsize};
		
		\node[shape = circle, draw, line width = 1pt, minimum size = 2.5mm, inner sep = 0mm] at (0,-2) {}
		child {node[shape = circle, draw, line width = 1pt, minimum size = 2.5mm, inner sep = 0mm] {}
		}
		child {node[shape = circle, draw, line width = 1pt, minimum size = 2.5mm, inner sep = 0mm] {}
		}
		child {node[shape = circle, draw, line width = 1pt, minimum size = 2.5mm, inner sep = 0mm] {}
		}
		child {node[shape = circle, draw, line width = 1pt, minimum size = 2.5mm, inner sep = 0mm] {}
		};
		
		\node (C) at (0, -2.5) {};
		\node (D) at (0, -3.5) {};
		
		\node (C2) at (7, -2.5) {};
		\node (D2) at (7, -3.5) {};
		
		\draw [->, line width=0.5mm, black] (C) -- (D);
		\draw [->, line width=0.5mm, black] (C2) -- (D2);
		
		\node at (7,-4) {\Large \(\mathcal T_3\) \normalsize};
		
		\node[shape = circle, draw, line width = 1pt, minimum size = 2.5mm, inner sep = 0mm] at (0,-4) {}
		child {node[shape = circle, draw, line width = 1pt, minimum size = 2.5mm, inner sep = 0mm] {}
			child {node[shape = circle, draw, line width = 1pt, minimum size = 2.5mm, inner sep = 0mm] {}}
			child {node[shape = circle, draw, line width = 1pt, minimum size = 2.5mm, inner sep = 0mm] {}}
			child {node[shape = circle, draw, line width = 1pt, minimum size = 2.5mm, inner sep = 0mm] {}}
			child {node[shape = circle, draw, line width = 1pt, minimum size = 2.5mm, inner sep = 0mm] {}}
		}
		child {node[shape = circle, draw, line width = 1pt, minimum size = 2.5mm, inner sep = 0mm] {}
		}
		child {node[shape = circle, draw, line width = 1pt, minimum size = 2.5mm, inner sep = 0mm] {}
		}
		child {node[shape = circle, draw, line width = 1pt, minimum size = 2.5mm, inner sep = 0mm] {}
		};
		
		\node (E) at (0, -4.5) {};
		\node (F) at (0, -5.5) {};
		
		\node (E2) at (7, -4.5) {};
		\node (F2) at (7, -5.5) {};
		
		\draw [->, line width=0.5mm, black] (E) -- (F);
		\draw [->, line width=0.5mm, black] (E2) -- (F2);
		
		\node at (7,-6) {\Large \(\mathcal T_4\) \normalsize};
		
		\node[shape = circle, draw, line width = 1pt, minimum size = 2.5mm, inner sep = 0mm] at (0, -6){}
		child {node[shape = circle, draw, line width = 1pt, minimum size = 2.5mm, inner sep = 0mm] {}
			child {node[shape = circle, draw, line width = 1pt, minimum size = 2.5mm, inner sep = 0mm] {}}
			child {node[shape = circle, draw, line width = 1pt, minimum size = 2.5mm, inner sep = 0mm] {}}
			child {node[shape = circle, draw, line width = 1pt, minimum size = 2.5mm, inner sep = 0mm] {}}
			child {node[shape = circle, draw, line width = 1pt, minimum size = 2.5mm, inner sep = 0mm] {}}
		}
		child {node[shape = circle, draw, line width = 1pt, minimum size = 2.5mm, inner sep = 0mm] {}
		}
		child {node[shape = circle, draw, line width = 1pt, minimum size = 2.5mm, inner sep = 0mm] {}
		}
		child {node[shape = circle, draw, line width = 1pt, minimum size = 2.5mm, inner sep = 0mm] {}
			child {node[shape = circle, draw, line width = 1pt, minimum size = 2.5mm, inner sep = 0mm] {}}
			child {node[shape = circle, draw, line width = 1pt, minimum size = 2.5mm, inner sep = 0mm] {}}
			child {node[shape = circle, draw, line width = 1pt, minimum size = 2.5mm, inner sep = 0mm] {}}
			child {node[shape = circle, draw, line width = 1pt, minimum size = 2.5mm, inner sep = 0mm] {}}
		};
		
		\end{tikzpicture}
	\end{adjustbox}
	\caption{Sequence of trees \(\left\{ \mathcal T_i \right\}_{i=1}^{4} \subseteq \mathcal T^{\mathcal Q}\) leading from \(\mathcal T_1 = \texttt{Root}(\mathcal T_{\mathcal W})\) to \(\mathcal T_4\).
		Note that \(\mathcal T_1 = \texttt{Root}\left( \mathcal T_{\mathcal W}\right)\) is the root node of \(\mathcal T_{\mathcal W}\) and \(\mathcal N\left( \mathcal T_{i+1}\right) \setminus \mathcal N\left( \mathcal T_{i}\right) = \mathcal C(n)\) for some \(n \in \mathcal N_{\text{leaf}}\left( \mathcal T_i \right)\) holds all \(i \in \{1,2,3\}\).}
	\label{fig:sequenceOfTrees}
\end{figure}
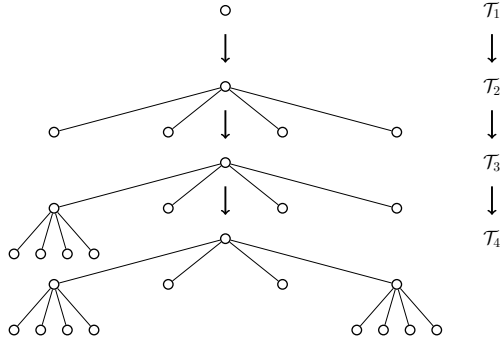

A key observation in formulating the problem \eqref{eq:IBproblemGeneral} over the space of multi-resolution trees is that each tree \(\mathcal T_q \in \mathcal T^{\mathcal Q}\) can be represented by a corresponding encoder of the form \(p^q(z|x)\), where \(p^q(z|x)\) specifies how the leaf nodes \(x \in \mathcal N_{\text{leaf}}\left(\mathcal T_{\mathcal W} \right)\) are mapped to nodes \(z \in \mathcal N_{\text{leaf}}\left( \mathcal T_q \right)\) to create the tree \( \mathcal T_q\) \cite{Larsson2020}. 
Thus, one may define the IB problem over the space of trees as 
\begin{equation}\label{eq:IBtreeProblem}
\mathcal T_{q^*} \in \argmax_{\mathcal T_q \in \mathcal T^{\mathcal Q}} L_Y(\mathcal T_q;\beta),
\end{equation}
where 
\begin{equation}
L_Y(\mathcal T_q;\beta) = K_Y(p^q(z|x);\beta),
\end{equation}
and 
\begin{equation} \label{eq:IBlagrangian}
K_Y(p(z|x);\beta) = I(Z;Y) - \frac{1}{\beta}I(Z;X).
\end{equation}
The problem \eqref{eq:IBtreeProblem} is difficult to solve for large world environments since the space of feasible multi-resolution trees can be large, making grid-search methods that exhaustively enumerate candidate solutions computationally intractable.

However, note that any tree \(\mathcal T_q \in \mathcal T^{\mathcal Q}\) can be obtained by starting at the tree \(\texttt{Root}(\mathcal T_\mathcal W) \in \mathcal T^{\mathcal Q}\) and performing a sequence of nodal expansions, as illustrated in Figure \ref{fig:sequenceOfTrees} \cite{Larsson2020}.
To capture this incremental change in \eqref{eq:IBlagrangian}, we define
\begin{equation}
\Delta L_Y(\mathcal T_i,\mathcal T_{i+1};\beta) = K_Y(p^{i+1}(z|x);\beta) - K_Y(p^i(z|x);\beta),
\end{equation}
which can be shown to only be a function of those nodes in \(\mathcal T_{i+1}\) that are merged to create the tree \(\mathcal T_i\) \cite{Larsson2020}.

Consequently, one can write the IB problem over the space of trees as 
\begin{equation} \label{eq:sequenicalIBLagragian}
\max_{m} \max_{\left\{\mathcal T_1,\ldots,\mathcal T_m \right\}} L_Y(\mathcal T_0;\beta) + \sum_{i=0}^{m-1} \Delta L_Y(\mathcal T_{i},\mathcal T_{i+1};\beta),
\end{equation}
where \(\mathcal T_0 \in \mathcal T^{\mathcal Q}\) is a given tree at which we start the nodal expansion process.
The tree \(\mathcal T_0\) is generally selected to be the tree \(\texttt{Root}(\mathcal T_{\mathcal W})\) so as to ensure that all trees in the space \(\mathcal T^{\mathcal Q}\) are obtainable via nodal expansion.
By changing the value of \(\beta > 0\), we alter the tree \(\mathcal T \in \mathcal T^{\mathcal Q}\), and thereby the corresponding multi-resolution representation of \(\mathcal W\), that emerges as a solution to \eqref{eq:sequenicalIBLagragian}.
Specifically, as \(\beta \to \infty\) we recover a tree \(\T \in \T^{\Q}\) that retains all the relevant information in \(\T_\W\) and as \(\beta \to 0\), we obtain the tree that maximizes compression (the root node), with a spectrum of solutions obtained for intermediate values of \(\beta > 0\).

Since the incremental change \(\Delta L_Y(\T_{i},\T_{i+1};\beta)\) does not depend on the entire configuration of the tree, one can compute the change in cost of expanding the node \(n \in \Nint(\T_\W)\), thereby removing \(n\) as a leaf node of the tree \(\T_i\) and adding \(\C(n)\) as leafs to create the tree \(\T_{i+1}\), according to 
\begin{equation} \label{eq:node-wiseDeltaL}
\Delta \tilde L_Y(n;\beta) = \Delta I_Y(n) - \frac{1}{\beta}\Delta I_X(n).
\end{equation}
In this relation, \(\Delta I_Y(\cdot)\) and \(\Delta I_X(\cdot)\) are non-negative functions that capture \(I(Z_{i+1};Y) - I(Z_i;Y)\) and \(I(Z_{i+1};X) - I(Z_i;X)\), respectively, where \(Z_{i}\) represents the compressed random variable encoded by the tree \(\T_{i}\), with \(Z_{i+1}\) defined analogously.
Through direct calculation, the functions \(\Delta I_Y(\cdot)\) and \(\Delta I_X(\cdot)\) can be shown to be dependent on the joint distribution \(p(x,y)\) through the Jensen-Shannon divergence measure and Shannon entropy, respectively \cite{Larsson2020,Slonim2000}.

The node-wise property in \eqref{eq:node-wiseDeltaL} is utilized by algorithms to solve the problem \eqref{eq:sequenicalIBLagragian}.
For example, a greedy approach inspects \(\Delta \tilde L_Y(n;\beta)\) in the set \(n \in \N_{\text{leaf}}(\T_i)\), and terminates if \(\Delta \tilde L_Y(n;\beta) \leq 0\) for all \(n \in \N_{\text{leaf}}(\T_{i})\) \cite{Larsson2020}.
While simple, such an approach does not generally find optimal solutions.
To overcome this, the authors of \cite{Larsson2020} define a function such that if \(n \in \Nint(\T_\W) \), then
\begin{equation}\label{eq:nodeWiseQfunction1}
\tilde Q_Y(n;\beta) = \max \{ \Delta \tilde L_Y(n;\beta) +  \sum_{n' \in \C(n)} \tilde Q_Y(n';\beta), ~ 0 \},
\end{equation}
and otherwise, if \(n \notin \Nint(\T_\W)\), then
\begin{equation}\label{eq:nodeWiseQfunction2}
\tilde Q_Y(n;\beta) = 0.
\end{equation}
The \(\tilde Q_Y\)-function incorporates reward of future expansions, analogously to Q-functions in reinforcement learning, and forms the basis for the so-called Q-tree search algorithm \cite{Larsson2020}.
Q-tree search then expands those nodes \(n \in \N_{\text{leaf}}(\T_i)\) of the current tree \(\T_i\) for which \(\tilde Q_Y(n;\beta) > 0\), continuing this process until it reaches a tree \(\T_{q^*}\) for which all nodes \(n \in \N_{\text{leaf}}(\T_{q^*})\) are such that \(\tilde Q_Y(n;\beta) \leq 0\).
We will make use of the Q-tree search algorithm to generate abstractions for the purpose of path planning, which we discuss next.

\subsection{Path-Planning on Abstractions} \label{sec:planningOnAbstractions}
Given a sequence of strictly increasing \(\beta > 0\), denoted by \(\left\{ \beta_i \right\}_{i=1}^N\), we generate a corresponding sequence of trees \(\left\{ \mathcal T_{\beta_i} \right\}_{i=1}^N\) by solving the information-theoretic problem in Section \ref{sec:ITIB_Abstractions} by employing the Q-tree search algorithm. 
A corresponding sequence of graphs \(\left\{ \mathcal G(\mathcal T_{\beta_i} ) \right\}_{i=1}^{N}\) can then be constructed, where each \(\mathcal G(\mathcal T_{\beta_i})\) for \(i\in \{1,\ldots,N\}\) represents a multi-resolution depiction of the environment \(\mathcal W\) with fewer vertices than \(\mathcal G(\mathcal T_{\mathcal W})\).
We will now use these reduced graphs to form approximations to Problem \ref{prob:frp_approx}, which brings us to the following definitions.
\begin{definition}[\hspace{-0.5pt}\cite{Larsson2020}] \label{def:treeLeaf} 
	Let \(n \in \mathcal N(\mathcal T)\) be a node in the tree \(\mathcal T\in \mathcal T^{\mathcal Q}\).  
	\emph{The subtree of} \(\mathcal T \in \mathcal T^{\mathcal Q}\) \emph{rooted at node} \(n\) is denoted by \(\mathcal T_{(n)}\) and has node set
		\begin{equation*}
		\mathcal N\left(\mathcal T_{(n)} \right) = \Big\{ n' \in \mathcal N(\mathcal T): n' \in  \bigcup_{i} \mathcal D_i \Big\},
		\end{equation*}
		where \(\mathcal D_1 = \left\{ n \right\}\), \(\mathcal D_{i+1} = \mathcal A \left( \mathcal D_i \right)\), and 
		\begin{equation*}
		\mathcal A\left( \mathcal D_i \right) = \Big\{ n' \in \mathcal N(\mathcal T_{\mathcal W}): n' \in \bigcup_{\hat n \in \mathcal{D}_i}\mathcal C\left( \hat n \right) \Big\}.
		\end{equation*}
\end{definition}
\begin{definition} 
	An \emph{abstract path} (AP) is a sequence of nodes \(\hat{\pi} = \left\{z_0,\ldots, z_R \right\} \subseteq \mathcal N_{\text{leaf}}\left( \mathcal T \right)\) for some \(\mathcal T\in \mathcal T^{\mathcal Q}\), \(\mathcal T \neq \mathcal T_{\mathcal W}\), such that  each \(z \in \hat\pi\) is distinct, the nodes \(z_0\) and \(z_R\) satisfy \(s_0 \in \mathcal N_{\text{leaf}}(\mathcal T_{\mathcal W(z_0)})\) and \(s_g \in \mathcal N_{\text{leaf}}(\mathcal T_{\mathcal W(z_R)})\), respectively, and if \(R > 0\) then \(z_i\), \(z_{i+1}\) are nodal neighbors for all \(i \in \{0,\ldots,R-1\}\).
	An \emph{\(\varepsilon\)-feasible abstract path} (\(\varepsilon\)-AP) is an AP \(\hat \pi\) such that \(\underset{z \in \hat \pi}{\bigcup}\mathcal N_{\text{leaf}}(\mathcal T_{\mathcal W(z)}) \subseteq \mathcal P_{\varepsilon}\).
\end{definition}

To obtain an AP requires the definition of a cost-function for abstracted representations.
This is challenging since the cost must: (i) be consistent with an FRP on the finest resolution; (ii) account for the cost of traversing aggregated nodes; and (iii) monotonically decrease with increased resolution, or equivalently, with increased \(\beta\).
The criterion (iii) above is to ensure that the paths \(\left\{ \hat{ \pi}_{\beta_i} \right\}_{i=1}^N\) represent approximations to an FRP \(\pi\) in the sense that the cost of a path \(\hat \pi_{\beta_i}\) should reduce to, and approach that of, an FRP \(\pi\) as \(\beta_i \to \infty\).

\begin{figure}[!t]
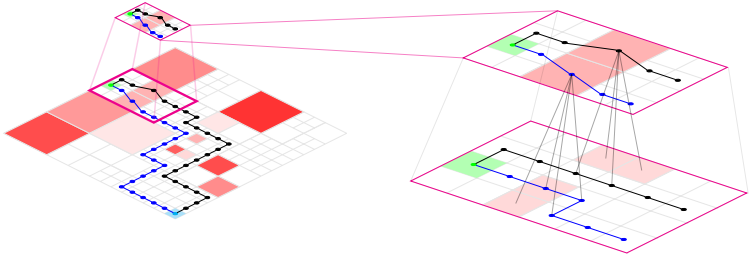

	\centering
	\begin{adjustbox}{max size={0.6\textwidth}}

	\end{adjustbox}
	\caption{Example OG of an environment \(\mathcal W\) with two AP \(\hat \pi_1\) (blue) and \(\hat \pi_2\) (black) leading from a given start location (cyan) to goal location (green).  For probabilistic obstacles (red), shading scales with the probability of occupancy.  Notice that both \(\hat \pi_1\) and \(\hat \pi_2\) pass through identical (adjacent) abstracted cells with non-zero probability of occupancy.  To determine the feasibility of these paths requires refinement, shown to the right.  Observe that, upon refinement, the path \(\hat \pi_1\) (blue) will be deemed infeasible, as it is not possible to traverse the left abstracted cell in the direction stipulated by \(\hat \pi_1\).  In contrast, the path \(\hat \pi_2\) (black) is feasible, since the right abstracted cell can be traversed in the direction required by \(\hat \pi_1\).  This shows the difficulty in guaranteeing path feasibility when planning on abstractions.  Our framework and definition of feasibility precludes this situation from occurring.}
	\label{fig:epsilonObsEx}
\end{figure}

To this end, we define \(V^{\lambda}_{\varepsilon}: \mathcal N\left( \mathcal T_{\mathcal W}\right) \to \Re_{+}\) as
\begin{equation} \label{eq:defOfAbstractV}
V^{\lambda}_{\varepsilon}(n) = \begin{cases}
c^{\lambda}_{\varepsilon}(n), & n \in \mathcal N_{\text{leaf}}\left( \mathcal T_{\mathcal W}\right), \\
%
\frac{1}{2^d}\sum_{n'\in\mathcal C(n)} V^{\lambda}_{\varepsilon}(n'), & \text{otherwise},
\end{cases}
\end{equation}
and consider the objective
\begin{equation} \label{eq:abstractObjectiveFunction}
\hat J^{\lambda}_{\varepsilon}(\hat \pi ; \beta) = \sum_{z \in \hat \pi} 2^{d r(z)}V^{\lambda}_{\varepsilon}(z).
\end{equation}   
Note that \(\hat J^{\lambda}_{\varepsilon}(\hat \pi ; \beta)\) depends on the trade-off parameter \(\beta > 0\), as \(\beta\) determines the tree \(\mathcal T_{\beta} \in \mathcal T^{\mathcal Q}\) on which the AP \(\hat \pi\) is planned.
Given \(\beta > 0\), we consider the problem
\begin{equation} \label{eq:abstractPathProblem}
\hat \pi_{\beta}^* \in \argmin_{\hat \pi \in \hat \Pi_{\beta}} \hat J^{\lambda}_{\varepsilon}(\hat \pi;\beta),
\end{equation} 
where \(\hat \Pi_\beta\) is the set of AP in \(\mathcal T_{\beta} \in \mathcal T^{\mathcal Q}\).

What remains to show is that the objective function value of \eqref{eq:abstractObjectiveFunction} monotonically decreases with increased \(\beta > 0\).
The following theorem establishes this result.
\begin{theorem} \label{thm:monotonicCost} 
	Let \(\varepsilon \in [0,1]\) and assume that there exists \(\beta_2 > \beta_1 > 0\) such that the corresponding trees \(\mathcal T_{\beta_1},\mathcal T_{\beta_2} \in \mathcal T^{\mathcal Q}\) satisfy \(\mathcal N \left(\mathcal T_{\beta_2}\right) \setminus \mathcal N \left(\mathcal T_{\beta_1}\right) = \mathcal C(n)\) for some \(n \in \mathcal N_{\text{leaf}}\left(\mathcal T_{\beta_1} \right)\).
	Furthermore, let \(\hat \pi^*_{\beta_1} \subseteq \mathcal N_{\text{leaf}}\left(\mathcal T_{\beta_1}\right)\) denote an abstract path in the tree \(\mathcal T_{\beta_1} \in \mathcal T^{\mathcal Q}\) satisfying \(\hat \pi^*_{\beta_1} \in \argmin_{\hat \pi_{\beta_1} \in \hat \Pi_{\beta_1}}\hat J^{\lambda}_{\varepsilon}(\hat \pi_{\beta_1};\beta_1)\).
	Then there exists an abstract path \(\hat \pi_{\beta_2} \subseteq \mathcal N_{\text{leaf}}\left(\mathcal T_{\beta_2}\right)\) such that \(\hat J^{\lambda}_{\varepsilon}(\hat \pi^*_{\beta_1};\beta_1) \geq \hat J^{\lambda}_{\varepsilon}(\hat \pi_{\beta_2};\beta_2)\).
\end{theorem}
\begin{proof}
The proof is presented in Appendix \ref{app:thmMonotoneCost}.
\end{proof}

By definition, \(\hat J^{\lambda}_{\varepsilon}(\hat \pi_{\beta_2};\beta_2) \geq \hat J^{\lambda}_{\varepsilon}(\hat \pi^*_{\beta_2};\beta_2)\) for all \(\hat \pi_{\beta_2} \in \hat \Pi_{\beta_2}\), and hence Theorem \ref{thm:monotonicCost} establishes that \(\hat J^{\lambda}_{\varepsilon}(\hat \pi^*_{\beta_1};\beta_1) \geq \hat J^{\lambda}_{\varepsilon}(\hat \pi^*_{\beta_2};\beta_2)\).
It should be noted that the same result holds even though two consecutive trees in the sequence \(\left\{\mathcal T_{\beta_i} \right\}_{i=1}^{N}\) do not necessarily satisfy \(\mathcal N \left(\mathcal T_{\beta_{i+1}}\right) \setminus \mathcal N \left(\mathcal T_{\beta_i}\right) = \mathcal C(n)\) for some \(n \in \mathcal N_{\text{leaf}}\left(\mathcal T_{\beta_i} \right)\).
This follows from the observation that any tree \(\mathcal T_m \in \mathcal T^{\mathcal Q}\) can be obtained from another sequence of trees, and so moving from \(\mathcal T_{\beta_{i}}\) to \(\mathcal T_{\beta_{i+1}}\) can be done by considering a sequence \(\left\{ \mathcal T_{k} \right\}_{k=0}^m\) where \(\mathcal T_{0} = \mathcal T_{\beta_i}\) and \(\mathcal T_{m} = \mathcal T_{\beta_{i+1}}\) for some \(m > 0\) where  \(\mathcal N \left(\mathcal T_{k+1}\right) \setminus \mathcal N \left(\mathcal T_{k}\right) = \mathcal C(n)\) for some \(n \in \mathcal N_{\text{leaf}}\left(\mathcal T_k \right)\) holds for all \(k \in \{0,\ldots,m-1\}\).

Theorem \ref{thm:monotonicCost} guarantees the monotonic improvement of the cost as a function of resolution.
It does not, however, guarantee that the cost of an AP converges to that of an FRP as \(\beta\to\infty\).
To address this, we require the following proposition.
\begin{proposition} \label{thm:convQtreeSearch}
Let \( \Delta I_Y: \Nint(\T_\W) \to [0,\infty) \) be the change in relevant information by expanding the node \(n \in \Nint(\T_\W)\).\footnote{See \cite{Larsson2020} for more information.}
Then the Q-tree search algorithm returns the tree \(\T_\W\) as \(\beta \to \infty\) if and only if \(\Delta I_Y(n) > 0\) for all \(n \in \N_{\ell-1}(\T_\W)\).
\end{proposition}
\begin{proof}
The proof is presented in Appendix \ref{app:proofThmConvQtreeSearch}.
\end{proof}
Theorem \ref{thm:monotonicCost} in conjunction with Proposition \ref{thm:convQtreeSearch} guarantee that the path cost sequence \(\{ \hat J^{\lambda}_{\varepsilon}(\hat \pi^*_{\beta_i};\beta_i)\}_{i=1}^N\) monotonically decreases and converges to \(J^{\lambda}_{\varepsilon}(\pi^*)\) as \(\beta_{N} \to \infty\).
We will now present a number of other properties of our problem, for which the following fact is useful.
\begin{fact} \label{fact:vLeafsRelation}
	Let \(u \in \{0,\ldots,\ell \}\), \(\varepsilon \in [0,1]\) and \(n \in r^{-1}(\{u\})\).
	Then \(V^{\lambda}_{\varepsilon}(n) = \frac{1}{2^{d r(n)}}\sum_{n'\in \mathcal N_{\text{leaf}}\left( \mathcal T_{\mathcal W(n)} \right)}V^{\lambda}_{\varepsilon}(n')\).
\end{fact}
\begin{proof}
The proof is presented in Appendix \ref{app:proofFctVleafsRelation}.
\end{proof}
\begin{proposition} \label{prop:epsFeasibleAbsPath}
	Let \(\varepsilon \in [0,1]\) and \(\beta > 0\).
	Then \(\hat J^{\lambda}_{\varepsilon}(\hat \pi;\beta) < M^{\lambda}_{\varepsilon}\) if and only if \(\hat \pi\) is a \(\varepsilon\)-feasible abstract path.
\end{proposition}
\begin{proof}
The proof is presented in Appendix \ref{app:ProofEpsFeasibleAbsPath}.
\end{proof}
\begin{corollary} \label{cor:epsFeasibleFRP}
	Let \(\varepsilon \in [0,1]\).
	Then \(J^{\lambda}_{\varepsilon}(\pi) < M_{\varepsilon}^{\lambda}\) if and only if \(\pi\) is a \(\varepsilon\)-feasible finest resolution path.
\end{corollary}
\begin{proof}
The proof is presented in Appendix \ref{app:ProofcorEpsFeasibleAbsPath}.
\end{proof}
The utility of Proposition \ref{prop:epsFeasibleAbsPath} and Corollary \ref{cor:epsFeasibleFRP} is that they provide conditions for quickly determining the feasibility of a path from knowledge of only the objective function value \(\hat J^{\lambda}_{\varepsilon}(\hat \pi;\beta)\) (or \(J^{\lambda}_{\varepsilon}(\pi)\)).
If the search terminates before an \(\varepsilon\)-feasible path has been found, the agent is provided with the most recent solution, which is guaranteed to be the least infeasible path available in the current tree.
Furthermore, as a result of Theorem \ref{thm:monotonicCost} and Proposition \ref{prop:epsFeasibleAbsPath}, if an AP \(\hat \pi^*_{\beta_{j}}\) is \(\varepsilon\)-feasible for some \(j \in \{1,\ldots, N\}\), then all AP in the sequence \(\{\hat \pi^*_{\beta_i} \}_{i=j}^{N}\) are also \(\varepsilon\)-feasible.
This ensures that the autonomous agent can never discover that a feasible path becomes infeasible with further refinement of the environment. 
An illustration is provided in Figure \ref{fig:epsilonObsEx}.
To conclude this section, we present the following proposition.
\begin{proposition} \label{prop:vEpsAndObs}
	Let \(\varepsilon \in [0,1]\) and \(n \in \mathcal N_{\text{int}}\left( \mathcal T_{\mathcal W} \right)\).
	Then \(V^{\lambda}_{\varepsilon}(n) > \lambda_1 + \varepsilon \lambda_2\) if and only if \(\mathcal N_{\text{leaf}}\left( \mathcal T_{\mathcal W(n)} \right) \cap \mathcal P^{c}_{\varepsilon} \neq \emptyset\).
\end{proposition}
\begin{proof}
The proof is presented in Appendix \ref{app:proofPropVepsAndObs}.
\end{proof}
Proposition \ref{prop:vEpsAndObs} allows an autonomous agent to quickly identify which leaf nodes in the tree \(\mathcal T \in \mathcal T^{\mathcal Q}\) are considered to be \(\varepsilon\)-obstacles and, consequently, which vertices in \(\mathcal G(\mathcal T)\) to avoid, if possible.
Next, we present a numerical example to demonstrate the utility of our approach.

\section{NUMERICAL EXAMPLE} \label{sec:numericalExample}

\begin{figure}[!t]
\centering
\subfloat[]{\includegraphics[width=0.24\textwidth]{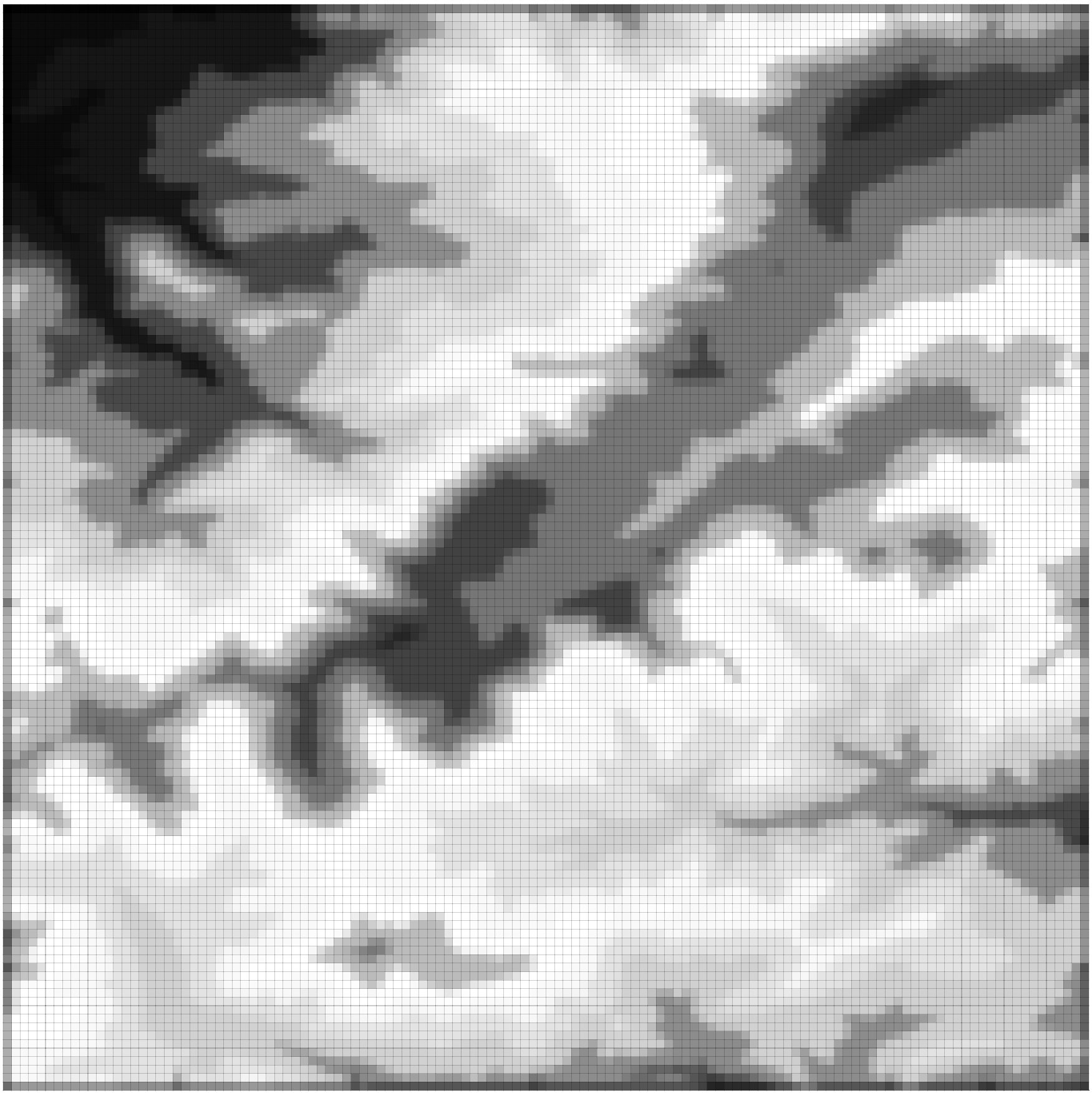}
\label{fig:originalW}}
\subfloat[]{\includegraphics[width=0.24\textwidth]{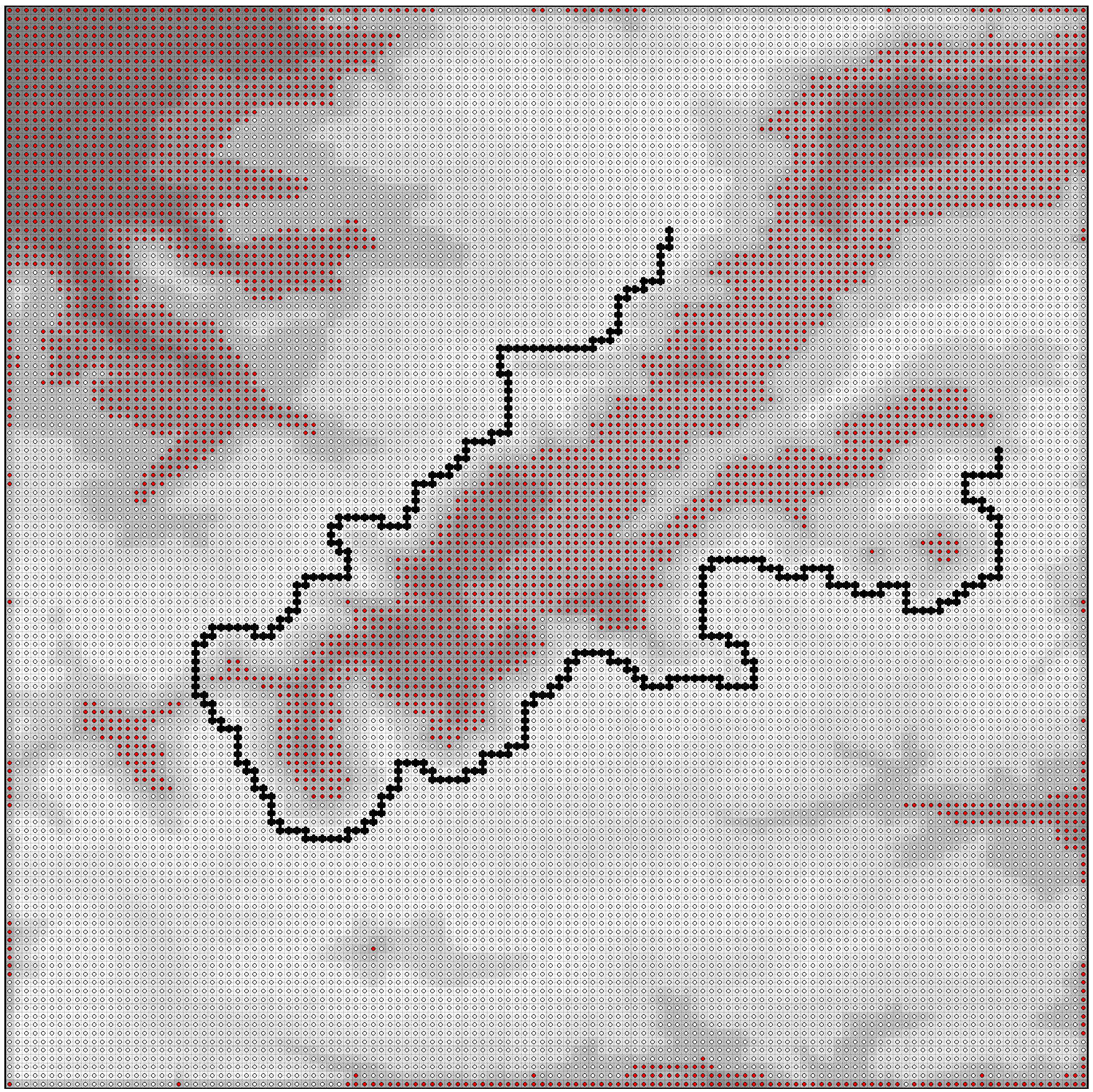}
\label{fig:originalWwithOPTFRP}}
\subfloat[]{\includegraphics[width=0.24\textwidth]{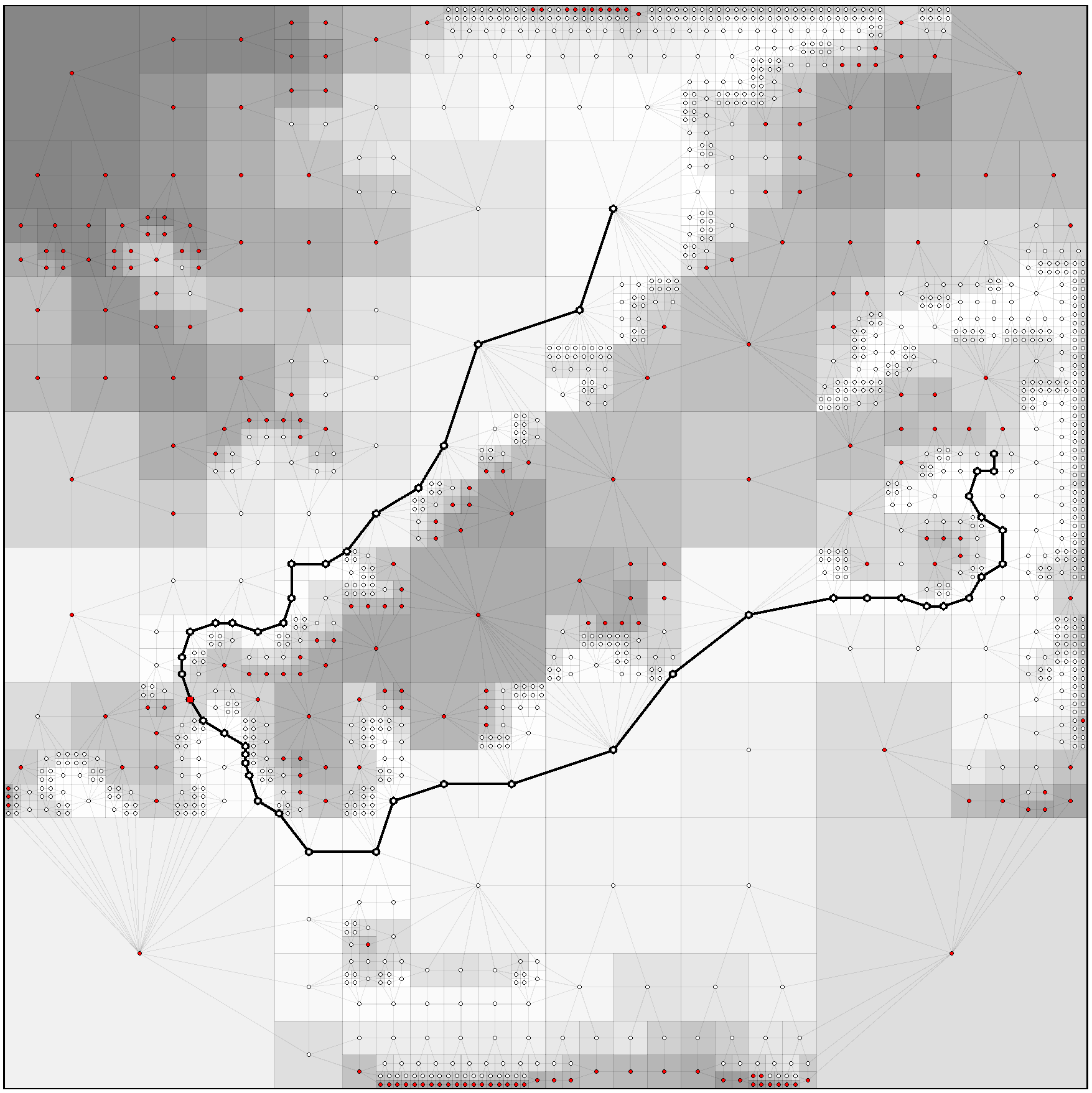}
\label{fig:pathGraph1}}
\subfloat[]{\includegraphics[width=0.24\textwidth]{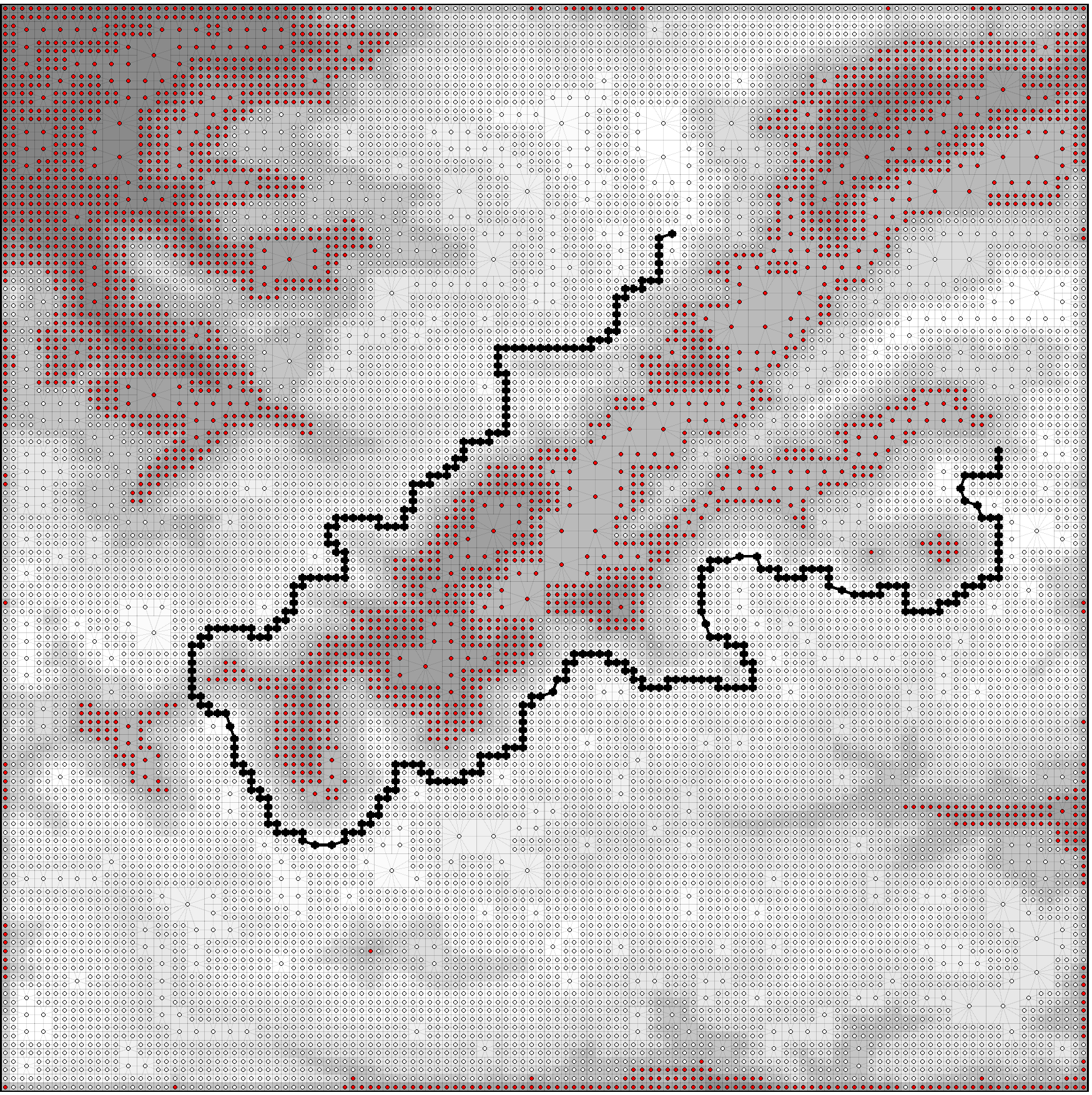}
\label{fig:pathGraph2}}
\caption{128\(\times\)128 environment \(\left( \lvert \Omega_X\rvert = 16384 \right)\) with graph abstraction and path examples for \(\varepsilon = 0.5\). Cost parameters are \(\lambda_1 = 0.001\) and \(\lambda_2 = 1\). Shading of grey scales with probability of occupancy. Red vertices are considered \(\varepsilon\)-obstacles as determined by Proposition \ref{prop:vEpsAndObs}. (a) original environment, (b) example FRP, (c) example AP and graph for \(\beta = 55 ~ \left( \%\lvert \Omega_X \rvert = 8.3\% \right)\), (d) example AP and graph for \(\beta = 1\times10^6 ~\left(\%\lvert \Omega_X \rvert = 83.4\% \right)\).}
\end{figure}

We consider the world \(\mathcal W\) to be given by the \(128 \times 128\) occupancy grid shown in Figure \ref{fig:originalW}.
The OG representation provides information regarding the conditional distribution \(p(y|x)\), whereby we then define the joint distribution \(p(y,x) = p(y|x)p(x)\) with \(p(x) = \nicefrac{1}{\lvert \mathcal N_{\text{leaf}}\left(\mathcal T_{\mathcal W}\right) \rvert}\) for all \(x \in \mathcal N_{\text{leaf}}\left(\mathcal T_{\mathcal W}\right)\).
By utilizing the uniform distribution \(p(x)\), we encode that the autonomous agent is equally likely to occupy any cell \(x \in \mathcal N_{\text{leaf}}\left(\mathcal T_{\mathcal W}\right)\) and will result in the IB method refining the environment in a region-agnostic manner \cite{Larsson2020}.
The joint distribution \(p(x,y)\), along with a sequence of strictly increasing positive values of \(\left\{ \beta_i \right\}_{i=1}^N\) are provided to the IB abstraction framework of Section \ref{sec:ITIB_Abstractions} to obtain the sequence of trees \(\left\{ \mathcal T_{\beta_i} \right\}_{i=1}^{N}\) along with the corresponding \(\left\{ \mathcal G\left( \mathcal T_{\beta_i} \right) \right\}_{i=1}^N\).
Given a start and goal location, the path planning problem \eqref{eq:abstractPathProblem} is then solved on each of the trees in the sequence \(\left\{ \mathcal T_{\beta_i} \right\}_{i=1}^{N}\) to obtain \(\{\hat \pi^*_{\beta_i} \}_{i=1}^N\).
Examples of abstract paths in the sequence \(\{\hat \pi^*_{\beta_i} \}_{i=1}^N\) are shown in Figures \ref{fig:pathGraph1} -- \ref{fig:pathGraph2} with a corresponding FRP shown in Figure \ref{fig:originalWwithOPTFRP}.

Figure \ref{fig:aveCostVsCompression} shows the average path cost ratio as a function of compression when averaging over 200 randomly sampled start and goal locations in the environment shown in Figure \ref{fig:originalW}.
For the averaging results, we ensure that the conditions of Proposition \ref{thm:convQtreeSearch} are satisfied to guarantee the path cost ratio converges to one as \(\beta \to \infty\).

From Figure \ref{fig:aveCostVsCompression}, we see that, on average, roughly \(15\%\) to \(18\%\) of the vertices of \(\mathcal G(\T_\W)\) are required in order to obtain a feasible path in the environment.
A reduced graph with only \(15\%\) to \(18\%\) of the vertices of \(\mathcal G(\T_\W)\) substantially reduces the computational effort required to find feasible solutions.
Figure \ref{fig:aveCostVsCompression} also shows that the path cost monotonically decreases with increased resolution.
The decreasing nature of the cost in Figure \ref{fig:aveCostVsCompression} is expected, since Theorem \ref{thm:monotonicCost} guarantees that the path cost between \emph{any} two points monotonically decreases as a function of resolution (or, equivalently, increased \(\beta > 0\)).
We also observe that the average path cost ratio converges to 1, corroborating the conditions for convergence set forth by Proposition \ref{thm:convQtreeSearch}.
Lastly, note that by utilizing a representation with approximately \(70\%\) of the nodes in \(\T_\W\) results, on average, in an abstract path \(\hat\pi^{*}_{\beta_i}\) for which \(\hat J_{\varepsilon}^{\lambda}(\hat \pi^*_{\beta_i};\beta_{i}) \) is within \(30\%\) of \(J^{\lambda}_{\varepsilon}(\pi^*)\).
We now provide some discussion of how our framework relates to bounded-rational decision making and anytime algorithms.
\begin{figure}[!t]
	\centering
	\includegraphics[width=0.5\textwidth]{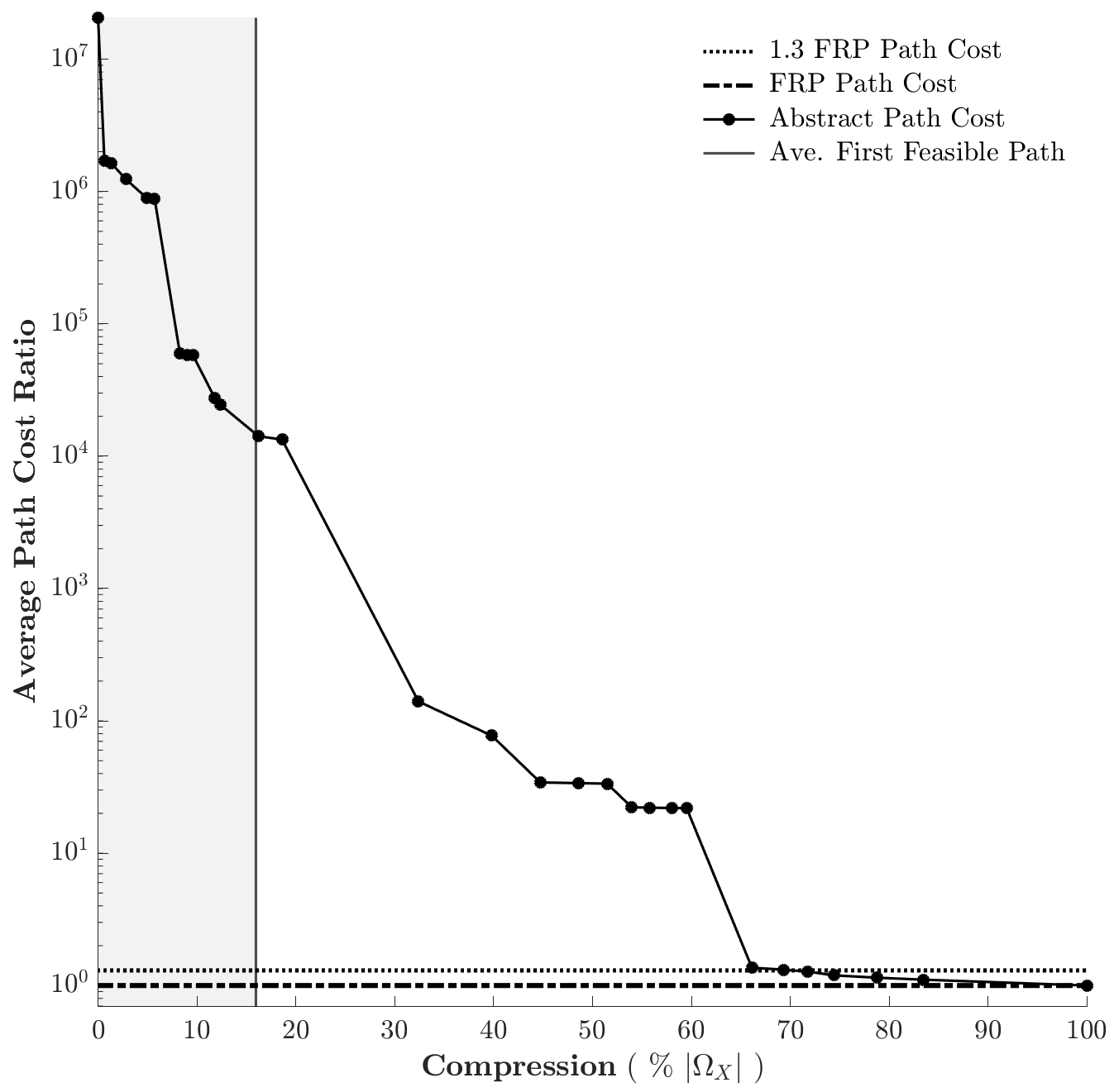}
	\caption{Logarithmic (base 10) value of average \(\nicefrac{\hat{J}^{\lambda}_{\varepsilon}(\hat \pi^*_{\beta_i};\beta_i)}{J_{\varepsilon}^{\lambda}(\pi^*)}\) versus compression for \(\varepsilon = 0.5\), \(\lambda_1 = 0.001\) and \(\lambda_2 = 1\). Note that \(\lvert \Omega_X \rvert = 16384\). Average values computed over 200 randomly sampled start and goal locations. Moving along the curve to the right is done by increasing \(\beta\). Average first feasible path line represents average compression at which first guaranteed feasible path in the abstracted environment is found.}
	\label{fig:aveCostVsCompression}
\end{figure}

\section{DISCUSSION} \label{sec:discussion}
The role of \(\beta > 0\) in our framework can be interpreted similarly to its role in other approaches for resource-constrained and bounded-rational decision making.
In fact, previous works, such as \cite{Larsson2017,Genewein2015,Tishby2010}, that study information-limited decision-making in stochastic domains have viewed the trade-off parameter \(\beta\) as a rationality parameter.
More specifically, these works formulate information-constrained MDPs by adding KL-divergence constraints to traditional MDP problems to constrain by what amount an optimal policy is permitted to differ from a policy provided to the agent beforehand.
In this view, the a priori policy is considered a default policy for which the agent resorts to in case it has no time or ability to discover a possibly better policy.
These studies argue that the resulting optimal policies for rational agents are recovered as \(\beta \to \infty\) and that policies for fully resource-limited agents are found as \(\beta \to 0\), with a spectrum of solutions for intermediate values of \(\beta > 0\).
Similarly, we see that \(\beta\) implicitly trades off the complexity of the search by influencing the number of vertices in the graph \(\mathcal G(\mathcal T_{\beta})\).
Furthermore, \(\beta\) not only trades the complexity of the search, but also changes the value of the resulting optimal path, since larger values of \(\beta\) lead to lower path costs, at the penalty of increased complexity.
In this way, we argue that our framework represents a bridge between bounded rationality and planning, providing a method for which one can trade path-optimality and complexity of the search directly, through a single parameter \(\beta\) and without the need to specify the abstractions a priori.

\begin{figure}[tbh]
	\centering
	\includegraphics[width=0.5\textwidth]{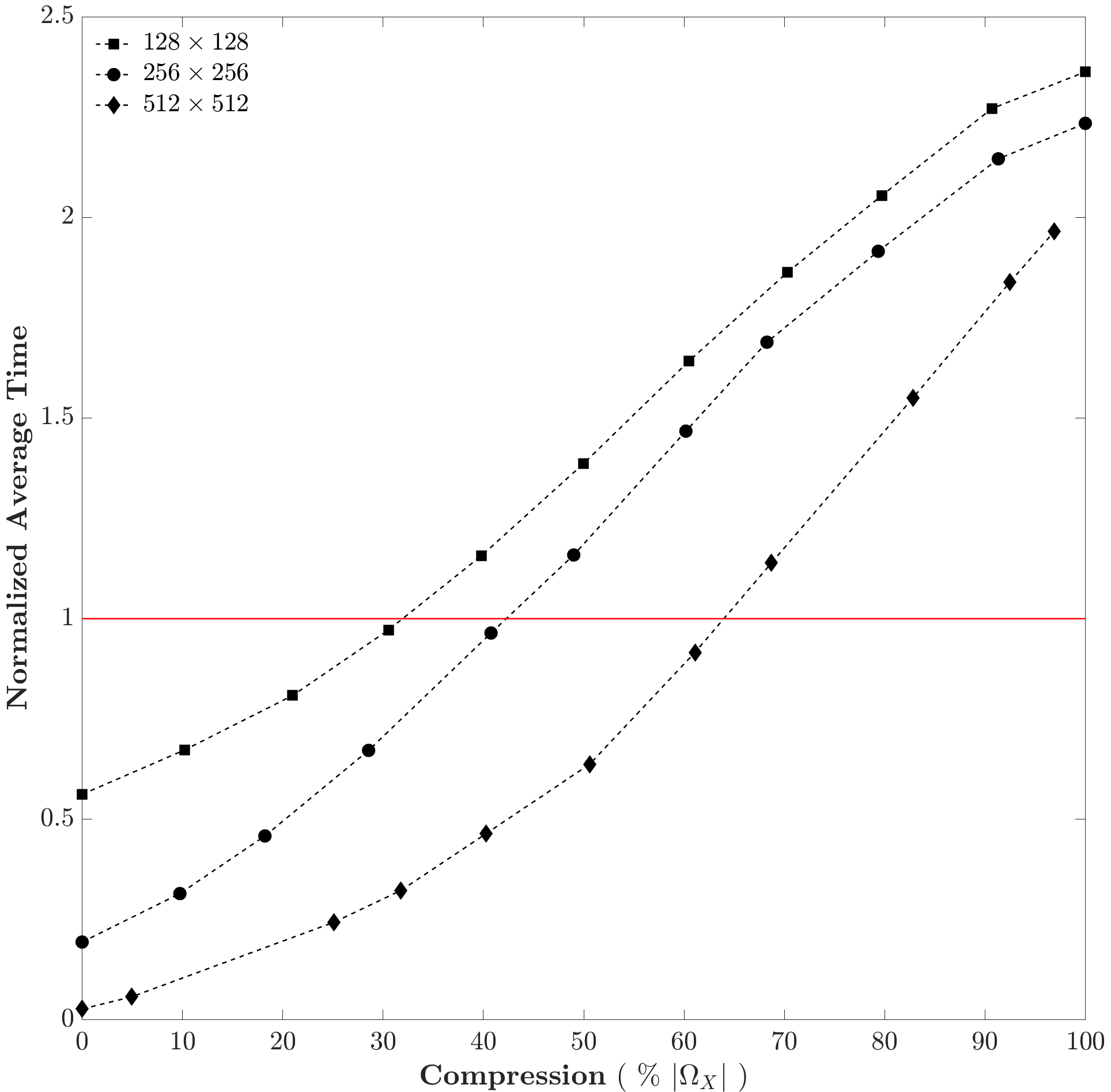}
	\caption{Average computation time \(\text{vs.}\) environment compression for two-dimensional grid sizes of \(128 \times 128\), \(256 \times 256\) and \(512 \times 512\). By increasing \(\beta\) one moves along the curves to the right. Timing results are obtained by averaging over 100 randomly selected start and goal locations. The y-axis represents average (normalized) time obtained by averaging the time for (i) information computation (population of \(\Delta I_Y(\cdot)\) and \(\Delta I_X (\cdot)\) values), (ii) computing Q-values (for each \(\beta\) according to \eqref{eq:nodeWiseQfunction1}-\eqref{eq:nodeWiseQfunction2}), (iii) running Q-tree search and (iv) employing Dijkstra graph-search to obtain an abstract path and normalizing by the average time to run Dijkstra on the finest-resolution map. Timing results assume the worst-case scenario that abstractions are generated from scratch for each value of \(\beta\) (or compression level) on a computer with a \(2.9\) GHz Intel \(\text{i}5\) CPU with \(8\) GB of RAM running MATLAB.}
	\label{fig:aveTimeVsCompression}
\end{figure}

As discussed, a number of other approaches exist that attempt to decrease the planning complexity of the problem by leveraging environment abstraction \cite{Cowlagi2008,Cowlagi2010,Cowlagi2011,Cowlagi2012,Cowlagi2012a,Tsiotras2007,Tsiotras2012,Kambhampati1986,Hauer2015,Hauer2016}.
While these approaches do provide a means to simplify the path-planning problem, they do not guarantee that the solution improves with increased deliberation, or planning, time.
We argue that our framework can provide a bridge between the planning time, the plan complexity and the plan cost by viewing the trade-off parameter \(\beta > 0\) as having a one-to-one correspondence with time \(t > 0\) by, for example, considering a strictly increasing mapping \(\Gamma(t) = \beta\), where \(\Gamma (\cdot)\) maps time \(t > 0\) to trade-off parameter values.
In this way, the IB as well as the planning problem \eqref{eq:abstractPathProblem} become time-dependent, and the improvement of the solution objective value with time is established by Theorem \ref{thm:monotonicCost}.
By viewing the problem in this way, we facilitate connections between planning complexity, optimality, and time, as suggested by anytime algorithms \cite{Dean1988}.

In regards to complexity and planning time, we note that the \(128 \times 128\) two-dimensional grid presented in Section \ref{sec:numericalExample} serves as an example to corroborate the theoretical results developed in this manuscript in a non-trivial setting.
In some real-world cases, it may happen that the computational benefits of leveraging abstractions for the purposes of planning are unlikely to outweigh the cost of executing Dijkstra directly on the finest resolution graph. 
It is known \cite{Larsson2020} that the complexity of the Q-tree search algorithm is \(\mathcal O(\lvert \N_{\text{leaf}}(\T_\W) \rvert )\) and so the worst-case complexity of executing Q-tree search and planning on the resulting compressed representation is \(\mathcal O(\lvert \N_{\text{leaf}}(\T_\W) \rvert) + \mathcal O(\lvert E(\T)\rvert + \lvert \mathcal{V}(\T) \rvert \log \lvert \mathcal{V}(\T) \rvert )\) for any \(\T\in\T^\Q\), as compared to \(\mathcal O(\lvert E(\T_\W)\rvert + \lvert \mathcal{V}(\T_\W) \rvert \log \lvert \mathcal{V}(\T_\W) \rvert )\) when solving the problem on the finest resolution representation.
While this theoretical result is useful, we also provide timing results obtained by executing our algorithm for various two-dimensional grid sizes in Figure \ref{fig:aveTimeVsCompression}.
Figure \ref{fig:aveTimeVsCompression} shows that, for all the two-dimensional grid sizes shown, there is a range for which our approach provides a computational benefit as compared to running Dijkstra on the finest resolution graph.
Moreover, comparing Figures \ref{fig:aveCostVsCompression} and \ref{fig:aveTimeVsCompression}, we see that for the \(128 \times 128\) grid example considered in Section \ref{sec:numericalExample}, a feasible path can be found (which requires approximately \(15\%\) of the nodes of \(\N_{\text{leaf}}(\T_\W)\) on average) at a lower computational cost than executing a search on the finest resolution.
While this path may not be optimal with respect to an FRP, it is computationally cheaper to obtain, and thus an agent requires, on average, less planning time to have a feasible plan as compared to executing a search on the finest resolution graph.
Additionally, for the \(512 \times 512\) two-dimensional grid, we see a computational cost savings by utilizing abstraction and planning up to a compression level of approximately \(65\%\) when comparing against the time required to obtain a finest resolution path, and that, as the number of vertices in \(\mathcal G(\T_\W)\) increases, the range over which there is a computational benefit to employing our approach grows, as evidenced in Figure \ref{fig:aveTimeVsCompression}.
Consequently, we hypothesize that the computational benefits of our approach will only increase as the dimension of the search space grows.
The computational savings come at the cost of a diminished performance as measured by the abstract cost, although this due to the intrinsic need to trade computational complexity and path optimality that we discuss throughout this paper.
Lastly, it is important to view Figure \ref{fig:aveTimeVsCompression} keeping in mind: (i) the results assume the worst-case scenario that the abstractions are not re-used and must be generated from scratch (Q-tree search is initialized at the root node of \(\T_\W\)) and are used to produce a single abstract path (\(\text{i.e.,}\) the abstractions are not re-used for multiple plans on the same abstraction) and (ii) the environment is two-dimensional.

Finally, recall that for a given value of \(\beta > 0\), the abstraction returned by Q-tree search will be a tree that retains the maximum amount of information regarding the relevant variable \(Y\) for a given level of compression.
This process is subject to the choice of the relevant variable, which in this paper was taken to be the cell occupancy.
Importantly, our framework holds for other choices of the relevant variable, albeit the average compression level to find the first feasible solution and the compression level for the average path cost to reach \(30\%\) of the FRP cost, as shown in Figure \ref{fig:aveCostVsCompression}, may change.
Investigating the selection of the relevant random variable and its implications is therefore of interest, and is a topic we leave for future work.

\section{CONCLUSIONS} \label{sec:conclusion}
In this paper, we have shown how a path-planning problem can be systematically simplified through the use of multi-resolution tree abstractions generated by an information-theoretic framework that can be tailored to agent resource limitations.
A number of theoretical results are presented that establish formal connections between the path quality, the graph-search complexity, and the information contained in the reduced graphs.
Our framework provides a principled way to generate abstractions tailored to agent resource constraints, while simultaneously providing guarantees on the monotonic improvement of the path cost as a function of environment resolution.
A non-trivial numerical example is presented to corroborate the theoretical findings and showcase the utility of the approach.
Lastly, we provided a discussion analyzing the interpretation of our framework within the real of bounded-rational decision making and anytime algorithms.


\bibliographystyle{IEEEtran}

\begin{thebibliography}{10}
	\providecommand{\url}[1]{#1}
	\csname url@samestyle\endcsname
	\providecommand{\newblock}{\relax}
	\providecommand{\bibinfo}[2]{#2}
	\providecommand{\BIBentrySTDinterwordspacing}{\spaceskip=0pt\relax}
	\providecommand{\BIBentryALTinterwordstretchfactor}{4}
	\providecommand{\BIBentryALTinterwordspacing}{\spaceskip=\fontdimen2\font plus
		\BIBentryALTinterwordstretchfactor\fontdimen3\font minus
		\fontdimen4\font\relax}
	\providecommand{\BIBforeignlanguage}[2]{{%
			\expandafter\ifx\csname l@#1\endcsname\relax
			\typeout{** WARNING: IEEEtran.bst: No hyphenation pattern has been}%
			\typeout{** loaded for the language `#1'. Using the pattern for}%
			\typeout{** the default language instead.}%
			\else
			\language=\csname l@#1\endcsname
			\fi
			#2}}
	\providecommand{\BIBdecl}{\relax}
	\BIBdecl
	
	\bibitem{Bertsekas2012}
	D.~P. Bertsekas, \emph{Dynamic Programming and Optimal Control}, 4th~ed.\hskip
	1em plus 0.5em minus 0.4em\relax Athena Scientific, 2012, vol.~2.
	
	\bibitem{Sutton2016}
	R.~S. Sutton and A.~G. Barto, \emph{Reinforcement Learning: An Introduction},
	2nd~ed.\hskip 1em plus 0.5em minus 0.4em\relax The MIT Press, 2016.
	
	\bibitem{Dean1988}
	T.~Dean and M.~Boddy, ``An analysis of time-dependent planning,'' in
	\emph{AAAI}, St. Paul, MN, August 21-26 1988, pp. 49--54.
	
	\bibitem{Kambhampati1986}
	S.~Kambhampati and L.~S. Davis, ``Multiresolution path planning for mobile
	robots,'' \emph{IEEE Journal of Robotics and Automation}, vol. RA-2, no.~3,
	pp. 135--145, September 1986.
	
	\bibitem{Tsiotras2012}
	P.~Tsiotras, D.~Jung, and E.~Bakolas, ``Multiresolution hierarchical
	path-planning for small {UAV}s using wavelet decompositions,'' \emph{Journal
		of Intelligent and Robotic Systems}, vol.~66, no.~4, pp. 505--522, June 2012.
	
	\bibitem{Behnke2004}
	S.~Behnke, ``Local multiresolution path planning,'' \emph{Lecture Notes in
		Artificial Intelligence}, vol. 3020, no.~1, pp. 332--343, 2004.
	
	\bibitem{Lu2012}
	Y.~Lu, X.~Huo, and P.~Tsiotras, ``A beamlet-based graph structure for path
	planning using multiscale information,'' \emph{IEEE Transactions on Automatic
		Control}, vol.~57, no.~5, pp. 1166 -- 1178, 2012.
	
	\bibitem{Tsiotras2007}
	P.~Tsiotras and E.~Bakolas, ``A hierarchical on-line path planning scheme using
	wavelets,'' in \emph{European Control Conference}, Kos, Greece, July 2-5
	2007.
	
	\bibitem{Cowlagi2008}
	R.~V. Cowlagi and P.~Tsiotras, ``Multiresolution path planning with wavelets: A
	local replanning approach,'' in \emph{American Control Conference}, Seattle,
	WA, June 11-13 2008, pp. 1220--1225.
	
	\bibitem{Cowlagi2010}
	------, ``Multi-resolution path planning: Theoretical analysis, efficient
	implementation, and extensions to dynamic environments,'' in \emph{IEEE
		Conference on Decision and Control}, Atlanta, GA, December 15-17 2010, pp.
	1384--1390.
	
	\bibitem{Cowlagi2011}
	R.~V. Cowlagi, ``Hierarchical motion planning for autonomous aerial and
	terrestrial vehicles,'' Ph.D. dissertation, Georgia Institute of Technology,
	2011.
	
	\bibitem{Cowlagi2012}
	R.~V. Cowlagi and P.~Tsiotras, ``Multiresolution motion planning for autonomous
	agents via wavelet-based cell decompositions,'' \emph{IEEE Transactions on
		Systems, Man and Cybernetics- Part B: Cybernetics}, vol.~42, no.~5, pp.
	1455--1469, October 2012.
	
	\bibitem{Cowlagi2012a}
	------, ``Hierarchical motion planning with kinodynamic feasibility
	guarantees,'' \emph{IEEE Transactions on Robotics}, vol.~28, no.~2, pp.
	379--395, April 2012.
	
	\bibitem{Hauer2015}
	F.~Hauer, A.~Kundu, J.~M. Rehg, and P.~Tsiotras, ``Multi-scale perception and
	path planning on probabilistic obstacle maps,'' in \emph{IEEE International
		Conference on Robotics and Automation}, Seattle, WA, May 26-30 2015, pp.
	4210--4215.
	
	\bibitem{Hauer2016}
	F.~Hauer and P.~Tsiotras, ``Reduced complexity multi-scale path-planning on
	probabilitic maps,'' in \emph{IEEE Conference on Robotics and Automation},
	Stockholm, Sweden, 16-21 May 2016, pp. 83--88.
	
	\bibitem{Kraetzschmar2004}
	G.~K. Kraetzschmar, G.~P. Gassull, and K.~Uhl, ``Probabilistic quadtrees for
	variable-resolution mapping of large environments,'' \emph{IFAC Proceedings
		Volumes}, vol.~37, no.~8, pp. 675--680, July 2004.
	
	\bibitem{Tishby2010}
	N.~Tishby and D.~Polani, \emph{Information Theory of Decisions and
		Actions}.\hskip 1em plus 0.5em minus 0.4em\relax Springer, NY, 2010, ch.~11,
	pp. 601--636.
	
	\bibitem{Ortega2011}
	P.~A. Ortega and D.~A. Braun, ``Information, utility and bounded rationality,''
	in \emph{International Conference on Artifical General Intelligence},
	Mountain View, CA, August 3-6 2011, pp. 269--274.
	
	\bibitem{Genewein2015}
	T.~Genewein, F.~Leibfried, J.~Grau-Moya, and D.~A. Braun, ``Bounded
	rationality, abstraction, and hierarchical decision-making: An
	information-theoretic optimality principle,'' \emph{Frontiers in Robotics and
		AI}, vol.~27, no.~2, pp. 1--24, November 2015.
	
	\bibitem{Larsson2017}
	D.~T. Larsson, D.~Braun, and P.~Tsiotras, ``Hierarchical state abstractions for
	decision-making problems with computational constraints,'' in \emph{IEEE
		Conference on Decision and Control}, Melbourne, Australia, December 12-15
	2017, pp. 1138--1143.
	
	\bibitem{Rubin2012}
	J.~Rubin, O.~Shamir, and N.~Tishby, ``Trading value and information in
	{MDPs},'' in \emph{Decision Making with Imperfect Decision Makers}.\hskip 1em
	plus 0.5em minus 0.4em\relax Springer-Verlag Heidelberg, 2012, ch.~3, pp. 57
	-- 74.
	
	\bibitem{Larsson2020}
	D.~T. Larsson, D.~Maity, and P.~Tsiotras, ``Q-tree search: An
	information-theoretic approach toward hierarchical abstractions for agents
	with computational limitations,'' \emph{IEEE Transactions on Robotics},
	vol.~36, no.~6, pp. 1669--1685, Dec. 2020.
	
	\bibitem{Cover2006}
	T.~M. Cover and J.~A. Thomas, \emph{Elements of Information Theory}, 2nd,
	Ed.\hskip 1em plus 0.5em minus 0.4em\relax John Wiley \& Sons, 2006.
	
	\bibitem{Tishby1999}
	N.~Tishby, F.~C. Pereira, and W.~Bialek, ``The information bottleneck method,''
	in \emph{The 37th Annual Allerton Conference on Communication, Control and
		Computing}, Monticello, IL, September 22-24 1999, pp. 368--377.
	
	\bibitem{Slonim2000}
	N.~Slonim and N.~Tishby, ``Agglomerative information bottleneck,'' in
	\emph{Advances in Neural Information Processing}, Denver, CO, November 28-30
	2000, pp. 617--623.
	
\end{thebibliography}


\begin{appendices}
%
\section{Proof of Theorem \ref{thm:monotonicCost}} \label{app:thmMonotoneCost}
The proof is given in two parts. 
We first present and provide a proof of the following lemma before providing the proof of Theorem \ref{thm:monotonicCost}.
\begin{lemma} \label{lemma:vProof1}
	Let \(n \in \mathcal N_{\text{int}} \left( \mathcal T_{\mathcal W} \right)\).
	Then \(2^{d r(n)}V^{\lambda}_{\varepsilon}(n) \geq \sum_{n' \in \mathcal C(n) \setminus\mathcal S} 2^{d r(n')}V^{\lambda}_{\varepsilon}(n')\) for all \(\mathcal S \subseteq \mathcal{C}(n)\).
\end{lemma}
\vspace{0.1cm}
\begin{proof}
The proof is given by contradiction.
Assume that \(2^{d r(n)}V^{\lambda}_{\varepsilon}(n) < \sum_{n' \in \mathcal C(n) \setminus \mathcal S} 2^{d r(n')}V^{\lambda}_{\varepsilon}(n')\) for some \(\mathcal S \subseteq \mathcal C(n)\).
This implies, since \(n' \in \mathcal C(n)\),
\begin{equation*}
2^{d (r(n')+1) }V^{\lambda}_{\varepsilon}(n) - \sum_{n' \in \mathcal C(n) \setminus \mathcal S } 2^{d r(n')} V^{\lambda}_{\varepsilon}(n') <  0.
\end{equation*}
Then,
\begin{equation*}
2^{d r(n')} \left( 2^d V^{\lambda}_{\varepsilon}(n) - \sum_{n' \in \mathcal C(n) \setminus \mathcal S } V^{\lambda}_{\varepsilon}(n') \right) < 0,
\end{equation*}
as, for all \(n' \in \mathcal C(n)\), \(2^{d r(n')}\) is a constant.
Furthermore, since \(2^{d r(n')} > 0\), we have
\begin{equation*}
2^d V^{\lambda}_{\varepsilon}(n) - \sum_{n' \in \mathcal C(n) \setminus \mathcal S} V^{\lambda}_{\varepsilon}(n') < 0,
\end{equation*}
which gives
\begin{equation*}
V^{\lambda}_{\varepsilon}(n) < 2^{-d}\sum_{n' \in \mathcal C(n) \setminus \mathcal S} V^{\lambda}_{\varepsilon}(n'),
\end{equation*}
for some \(\mathcal S \subseteq \mathcal C(n)\).
However, as \(0 \leq V^{\lambda}_{\varepsilon}(n')\) for all \(n' \in \mathcal C(n)\), we have that 
\begin{equation*}
V^{\lambda}_{\varepsilon}(n) < 2^{-d}\sum_{n' \in \mathcal C(n) \setminus \mathcal S} V^{\lambda}_{\varepsilon}(n') \leq 2^{-d}\sum_{n' \in \mathcal C(n) } V^{\lambda}_{\varepsilon}(n') = V^{\lambda}_{\varepsilon}(n),
\end{equation*}
leading to a contradiction.
\end{proof}

\vspace{0.2cm}
\noindent We now prove Theorem \ref{thm:monotonicCost}.
\vspace{0.2cm}

\begin{proof}
The proof is given by construction.
There are two cases to consider: \(\hat \pi^*_{\beta_1} \cap \{ n \} = \varnothing\), and \(\hat \pi^*_{\beta_1} \cap \{ n \} \neq \varnothing\).

\vspace{0.3cm}

\noindent
First consider the case \(\hat \pi^*_{\beta_1} \cap \{ n \} = \emptyset\).  
It follows,
\begin{equation*}
\hat \pi^*_{\beta_1} \subseteq \mathcal N_{\text{leaf}}\left(\mathcal T_{\beta_1}\right) \cap \mathcal N_{\text{leaf}}\left(\mathcal T_{\beta_2}\right) \subset \mathcal N_{\text{leaf}}\left( \mathcal T_{\beta_2} \right),
\end{equation*}
and thus \(\hat \pi^*_{\beta_1} \subset \mathcal N_{\text{leaf}}\left( \mathcal T_{\beta_2} \right)\).
Take \(\hat \pi_{\beta_2} = \hat \pi^*_{\beta_1}\) and consequently \(\hat J^{\lambda}_{\varepsilon}(\hat \pi^*_{\beta_1};\beta_1) = \hat J^{\lambda}_{\varepsilon}(\hat \pi_{\beta_2};\beta_2)\).
Hence, there exists a path \(\hat \pi_{\beta_2} \subseteq \mathcal N_{\text{leaf}}\left(\mathcal T_{\beta_2}\right)\) such that \(\hat J^{\lambda}_{\varepsilon}(\hat \pi^*_{\beta_1};\beta_1) \geq \hat J^{\lambda}_{\varepsilon}(\hat \pi_{\beta_2};\beta_2)\).

\vspace{0.3cm}

\noindent
Now consider \(\hat \pi^*_{\beta_1} \cap \{ n \} \neq \emptyset\).
In this case, without loss of generality, \(\hat \pi^*_{\beta_1} = \{z_0,\ldots, z_{i-1}, z_{i}, z_{i+1},\ldots, z_R\} \subseteq \N_{\text{leaf}}(\T_{\beta_1})\) is an abstract path where \(z_i = n\).
Since the node \(n\) is expanded, we re-route the path through the children of \(n\).
To this end, consider \(\hat \pi_{\beta_2} = \{z_0,\ldots, z_{i-1}, z'_{i_1},\ldots,z'_{i_u}, z_{i+1},\ldots, z_R\}\), where \(\{z'_{i_1},\ldots,z'_{i_u}\} \subseteq \C(n)\) is a sequence of nodes so as to render \(\hat \pi_{\beta_2}\) an abstract path.
Notice that \(\hat \pi_{\beta_2} \subseteq \N_{\text{leaf}}(\T_{\beta_2})\), and 
\begin{align*}
\hat J^{\lambda}_{\varepsilon}(\hat \pi^*_{\beta_1} ;\beta_1 ) - \hat J^{\lambda}_{\varepsilon}(\hat \pi_{\beta_2};\beta_2 ) &= \sum_{z \in \hat \pi^*_{\beta_1}}2^{d r(z)}V^{\lambda}_{\varepsilon}(z) - \sum_{z \in \hat \pi_{\beta_2}} 2^{d r(z)}V^{\lambda}_{\varepsilon}(z), \\
&= 2^{d r(n)}V^{\lambda}_{\varepsilon}(n) - \sum_{n' \in \{z'_{i_1},\ldots,z'_{i_u}\}}2^{d r(n')}V^{\lambda}_{\varepsilon}(n').
\end{align*}
Since \(\{z'_{i_1},\ldots,z'_{i_u}\} \subseteq \C(n)\), it follows from Lemma \ref{lemma:vProof1} that 
\begin{equation*}
2^{d r(n)}V^{\lambda}_{\varepsilon}(n) - \sum_{n' \in \{z'_{i_1},\ldots,z'_{i_u}\}}2^{d r(n')}V^{\lambda}_{\varepsilon}(n') \geq 0,
\end{equation*}
and thus \(\hat J^{\lambda}_{\varepsilon}(\hat \pi^*_{\beta_1} ;\beta_1 ) \geq \hat J^{\lambda}_{\varepsilon}\left(\hat \pi_{\beta_2};\beta_2 \right)\).
Therefore, there exists a path \(\hat \pi_{\beta_2} \subseteq \mathcal N_{\text{leaf}}\left( \mathcal T_{\beta_2} \right)\) such that \(\hat J^{\lambda}_{\varepsilon}(\hat \pi^*_{\beta_1} ;\beta_1 ) \geq \hat J^{\lambda}_{\varepsilon}(\hat \pi_{\beta_2};\beta_2 )\). 
\end{proof}
%

\section{Proof of Proposition \ref{thm:convQtreeSearch}} \label{app:proofThmConvQtreeSearch}
The proof of Proposition \ref{thm:convQtreeSearch} depends on a number of results that must be established before we prove the proposition.
In what follows, we define the function \(\tilde P_Y: \N(\T_\W) \times \Re_{++} \to \Re\) as
\begin{equation*}
\tilde P_Y(n;\beta) = 
\begin{cases}
0, & \text{if } n \in \N_{\text{leaf}}(\T_\W), \\
\Delta \tilde L_Y(n;\beta) + \sum_{n'\in \C(n)} \tilde P_Y(n';\beta), & \text{ otherwise}.
\end{cases}
\end{equation*}
We now present the following results, which are used to prove Proposition \ref{thm:convQtreeSearch}.

\begin{lemma} \label{lem:limitPandRelationToDeltaL}
Let \(n \in \Nint(\T_\W)\) and \(\beta > 0\).
Then \(\tilde P_Y(n;\beta) = \sum_{s \in \Nint(\T_{\W(n)})} \Delta \tilde L_Y(s;\beta)\) and \(\underset{\beta \to \infty}{\lim}\tilde P_Y(n;\beta) = \sum_{s \in \Nint(\T_{\W(n)})} \Delta I_Y(s)\).
\end{lemma}
\begin{proof}
The proof is given by induction.
Consider any \(n \in \N_{\ell-1}(\T_\W)\) and notice that for all \(n \in \N_{\ell-1}(\T_\W)\), \(\Nint(\T_{\W(n)}) = \{n\}\).
Therefore, from the definition of \(\tilde P_Y\), we have
\begin{equation*}
\tilde P_Y(n;\beta) = \Delta \tilde L_Y(n;\beta) = \sum_{s \in \Nint(\T_{\W(n)})}\Delta \tilde L_Y(s;\beta).
\end{equation*}
Assume the lemma holds for all \(n' \in \N_{k+1}(\T_\W)\) and consider any \(n \in \N_{k}(\T_\W)\) for \(k \in \{0,\ldots,\ell-2\}\).
Then, from definition of \(\tilde P_Y\) and the induction hypothesis,
\begin{align*}
\tilde P_Y(n;\beta) &= \Delta \tilde L_Y(n;\beta) + \sum_{n' \in \C(n)} \tilde P_Y(n';\beta), \\
&= \Delta \tilde L_Y(n;\beta) + \sum_{n' \in \C(n)} \sum_{s \in  \Nint(\T_{\W(n')})}\Delta \tilde L_Y(s;\beta), \\
&= \Delta \tilde L_Y(n;\beta) + \sum_{s \in  \underset{{n' \in \C(n) }}{\bigcup}\Nint(\T_{\W(n')})}\Delta \tilde L_Y(s;\beta).
\end{align*}
Now, notice that \(\Nint(\T_{\W(n)}) = \{n \} \cup \left(\underset{n'\in\C(n)}{\bigcup} \Nint(\T_{\W(n')}) \right)\), and consequently,
\begin{align*}
\tilde P_Y(n;\beta) &=\Delta \tilde L_Y(n;\beta) + \sum_{s \in  \underset{n' \in \C(n)}{\bigcup}\Nint(\T_{\W(n')})}\Delta \tilde L_Y(s;\beta),\\
&= \sum_{s \in  \Nint(\T_{\W(n)})}\Delta \tilde L_Y(s;\beta).
\end{align*}
This proves the first part of the proposition.

\vspace{0.3cm}

\noindent
Now, to evaluate the limit as \(\beta \to \infty\), we make note that, from the definition of \(\Delta \tilde L_Y(\cdot;\beta)\), 
\begin{equation*}
\lim_{\beta \to \infty}\Delta \tilde L_Y(n;\beta) = \Delta I_Y(n),
\end{equation*}
for all \(n \in \Nint(\T_\W)\).
Therefore \(\underset{\beta \to \infty}{\lim}\Delta \tilde L_Y(n;\beta)\) exists for all \(n \in \Nint(\T_\W)\), and
\begin{equation*}
\lim_{\beta \to \infty} \tilde P_Y(n;\beta) = \sum_{s \in  \Nint(\T_{\W(n)})} \lim_{\beta \to \infty} \Delta \tilde L_Y(s;\beta) = \sum_{s \in  \Nint(\T_{\W(n)})} \Delta I_Y(s),
\end{equation*}
where \(\Nint(\T_{\W(n)})\) is a finite set for all \(n \in \Nint(\T_\W)\).
Thus, we have established \(\lim_{\beta \to \infty}\tilde P(n;\beta) = \sum_{s \in  \Nint(\T_{\W(n)})} \Delta I_Y(s)\) for all \(n \in \Nint(\T_\W)\).
\end{proof}

\begin{fact} \label{fct:PlowerboundQ}
Let \(n \in \Nint(\T_\W)\) and \(\beta > 0\).
Then \(\tilde P_Y(n;\beta) \leq \tilde Q_Y(n;\beta)\).
\end{fact}
\begin{proof}
The proof is given by induction.
Consider any \(n \in \N_{\ell-1} (\T_\W)\).
From the definition of \(\tilde P_Y(\cdot;\beta)\), we have
\begin{equation*}
\tilde P_Y(n;\beta) = \Delta \tilde L_Y(n;\beta) \leq \max\{\Delta \tilde L_Y(n;\beta), ~0 \} = \tilde Q_Y(n;\beta).
\end{equation*}
Assume the result holds for all \(n' \in \N_{k+1}(\T_\W)\) and consider any \(n \in \N_k(\T_\W)\) for \(k \in \{0,\ldots,\ell-2\}\).
Then,
\begin{equation*}
\tilde P_Y(n;\beta) = \Delta \tilde L_Y(n;\beta) + \sum_{n' \in \C(n)} \tilde P_Y(n';\beta) \leq \Delta \tilde L_Y(n;\beta) + \sum_{n' \in \C(n)} \tilde Q_Y(n';\beta),
\end{equation*}
where the inequality follows from the induction hypothesis as \(\C(n) \subseteq \N_{k+1}(\T_\W)\).
Consequently,
\begin{equation*}
\tilde P_Y(n;\beta) \leq \Delta \tilde L_Y(n;\beta) + \sum_{n' \in \C(n)} \tilde Q_Y(n';\beta) \leq \max\left \{ \Delta \tilde L_Y(n;\beta) + \sum_{n' \in \C(n)} \tilde Q_Y(n';\beta),~0\right \} = \tilde Q_Y(n;\beta),
\end{equation*}
establishing \(\tilde P_Y(n;\beta) \leq \tilde Q_Y(n;\beta)\) for all \(n \in \Nint(\T_\W)\).
\end{proof}

\begin{lemma} \label{lem:QupperBound}
Let \(n \in \Nint(\T_\W)\) and \(\beta > 0\).
Then \(\tilde Q_Y(n;\beta) \leq \sum_{s \in \Nint(\T_{\W(n)})} \Delta I_Y(s)\).
\end{lemma}
\begin{proof}
We will first prove that \(\Delta \tilde L_Y(n;\beta) + \sum_{n'\in\C(n)} \tilde Q_Y(n';\beta) \leq \sum_{s \in \Nint(\T_{\W(n)})}\Delta I_Y(s)\) for all \(n \in \Nint(\T_\W)\) and \(\beta > 0\).

\vspace{0.3cm}

\noindent
The proof is given by induction.
Consider any \(n \in \N_{\ell-1}(\T_\W)\).
Then, since \(\C(n) \subseteq \N_{\text{leaf}}(\T_\W)\), 
\begin{align*}
\Delta \tilde L_Y(n;\beta) + \sum_{n'\in\C(n)} \tilde Q_Y(n';\beta) &= \Delta \tilde L_Y(n;\beta),\\
&= \Delta I_Y(n) - \frac{1}{\beta} \Delta I_X(n) \leq \Delta I_Y(n)  = \sum_{s \in \Nint(\T_{\W(n)} )}\Delta I_Y(s),
\end{align*}
as \(\Nint(\T_{\W(n)}) = \{n\}\) for all \(n \in \N_{\ell-1}(\T_\W)\).

\vspace{0.3cm}

\noindent 
Now assume the result holds for all \(n' \in \N_{k+1}(\T_\W)\) and consider any \(n \in \N_k(\T_\W)\) for \(k \in \{0,\ldots,\ell-2\}\).
Then,
\begin{align*}
\Delta \tilde L_Y(n;\beta) + \sum_{n'\in\C(n)} \tilde Q_Y(n';\beta) &= \Delta \tilde L_Y(n;\beta) + \sum_{n'\in\C(n) \cap \mathcal R} \tilde Q_Y(n';\beta),\\
&= \Delta \tilde L_Y(n;\beta) + \sum_{n'\in\C(n) \cap \mathcal R} \left( \Delta \tilde L_Y(n';\beta) + \sum_{n'' \in \C(n')} \tilde Q(n'';\beta) \right) 
\end{align*}
where \(\mathcal R = \left\{n'\in \C(n) : \tilde Q_Y(n';\beta) > 0 \right\}\) and note that \(\tilde Q_Y(n';\beta) = 0\) for \(n' \in \C(n) \setminus \mathcal R\).
Since \(n' \in \C(n) \subseteq \N_{k+1}(\T_\W)\), we have from the induction hypothesis that 
\begin{equation*}
\Delta \tilde L_Y(n';\beta) + \sum_{n'' \in \C(n')} \tilde Q_Y(n'';\beta) \leq \sum_{s \in \Nint(\T_{\W(n')})}\Delta I_Y(s),
\end{equation*}
leading to
\begin{equation*}
\Delta \tilde L_Y(n;\beta) + \sum_{n'\in\C(n)} \tilde Q_Y(n';\beta) \leq \Delta \tilde L_Y(n;\beta) + \sum_{n'\in\C(n) \cap \mathcal R} \sum_{s \in \Nint(\T_{\W(n')})} \Delta I_Y(s).
\end{equation*}
Now, 
\begin{align*}
\Delta \tilde L_Y(n;\beta) +& \sum_{n'\in\C(n) \cap \mathcal R} \sum_{s \in \Nint(\T_{\W(n')})} \Delta I_Y(s) = \Delta \tilde L_Y(n;\beta) + \sum_{s \in \underset{n'\in\C(n) \cap \mathcal R}{\bigcup} \Nint(\T_{\W(n')})} \Delta I_Y(s),\\
&\leq \Delta \tilde L_Y(n;\beta) + \sum_{s \in \underset{n'\in\C(n) \cap \mathcal R}{\bigcup} \Nint(\T_{\W(n')})} \Delta I_Y(s) + \sum_{s \in \underset{n'\in\C(n) \setminus \mathcal R}{\bigcup} \Nint(\T_{\W(n')})} \Delta I_Y(s), \\
&\leq \Delta I_Y(n)+ \sum_{s \in \underset{n'\in\C(n) \cap \mathcal R}{\bigcup} \Nint(\T_{\W(n')})} \Delta I_Y(s) + \sum_{s \in \underset{n'\in\C(n) \setminus \mathcal R}{\bigcup} \Nint(\T_{\W(n')})} \Delta I_Y(s),\\
&= \sum_{s \in \Nint(\T_{\W(n)})} \Delta I_Y(s),
\end{align*}
where the first inequity follows from the non-negativity of \(\Delta I_Y(\cdot)\) and the second from the fact that \(\Delta \tilde L_Y(n;\beta) \leq \Delta I_Y(n)\) for all \(n \in \Nint(\T_\W)\).
Thus, \(\Delta \tilde L_Y(n;\beta) + \sum_{n'\in\C(n)}\tilde Q_Y(n';\beta) \leq \sum_{s \in \Nint(\T_{\W(n)})} \Delta I_Y(s)\) for all \(n \in \Nint(\T_\W)\).

\vspace{0.3cm}

\noindent
Consequently, we note that 
\begin{align*}
\tilde Q_Y(n;\beta) &= \max\left\{\Delta \tilde L_Y(n;\beta) + \sum_{n' \in \C(n)} \tilde Q_Y(n';\beta),~0\right\},\\
&\leq \max\left\{\sum_{s \in \Nint(\T_{\W(n)})} \Delta I_Y(s), ~0\right\},\\
&= \sum_{s \in \Nint(\T_{\W(n)})} \Delta I_Y(s),
\end{align*}
since \(\sum_{s \in \Nint(\T_{\W(n)})} \Delta I_Y(s) \geq 0\).
Therefore, \(\tilde Q_Y(n;\beta) \leq \sum_{s \in \Nint(\T_{\W(n)})} \Delta I_Y(s)\) for all \(n \in \Nint(\T_\W)\).
\end{proof}

%
\begin{lemma} \label{thm:limitOfQFunction}
Let \(n \in \mathcal N_{\text{int}}(\T_\W)\). 
Then \(\underset{\beta \to \infty}{\lim} \tilde Q_Y(n;\beta) = \sum_{s \in \mathcal N_{\text{int}}(\T_{\W(n)})}\Delta I_Y(s)\).
\end{lemma}
%

\begin{proof}
From Fact \ref{fct:PlowerboundQ} and Lemma \ref{lem:QupperBound}, we have,
\begin{equation*}
\tilde P_Y(n;\beta) \leq \tilde Q_Y(n;\beta) \leq  \sum_{s \in \Nint(\T_{\W(n)})} \Delta I_Y(s).
\end{equation*}
Taking the limit as \(\beta \to \infty\) results in
\begin{equation*}
\lim_{\beta \to \infty} \tilde P_Y(n;\beta) \leq \lim_{\beta \to \infty} \tilde Q_Y(n;\beta) \leq \lim_{\beta \to \infty}\sum_{s \in \Nint(\T_{\W(n)})} \Delta I_Y(s),
\end{equation*}
where from Lemma \ref{lem:limitPandRelationToDeltaL} and the fact that \(\sum_{s \in \Nint(\T_{\W(n)})} \Delta I_Y(s)\) does not depend on \(\beta\), we obtain
\begin{equation*}
\sum_{s \in \Nint(\T_{\W(n)})} \Delta I_Y(s) \leq \lim_{\beta \to \infty} \tilde Q_Y(n;\beta) \leq \sum_{s \in \Nint(\T_{\W(n)})} \Delta I_Y(s).
\end{equation*}
Therefore, \(\underset{\beta \to \infty}{\lim} \tilde Q_Y(n;\beta) = \sum_{s \in \Nint(\T_{\W(n)})} \Delta I_Y(s)\).
\end{proof}

\noindent
We now prove Proposition \ref{thm:convQtreeSearch}.
\begin{proof}
Assume the Q-tree search algorithm returns the tree \(\T_\W\) as \(\beta \to \infty\).
This implies all nodes \(n \in \Nint(\T_\W)\) are expanded as \(\beta \to \infty\) and therefore \(\underset{\beta \to \infty}{\lim}\tilde Q_Y(n;\beta) > 0\).
From Lemma \ref{thm:limitOfQFunction},
\begin{equation*}
\lim_{\beta \to \infty} \tilde Q_Y(n;\beta) = \sum_{s \in \Nint(\T_{\W(n)})}\Delta I_Y(s),
\end{equation*}
it follows that \(\sum_{s \in \Nint(\T_{\W(n)})}\Delta I_Y(s) > 0\) for all \(n \in \Nint(\T_\W)\).
Observe that \(\N_{\ell-1}(\T_\W) \subseteq \Nint(\T_\W)\) and that for all \(n \in \N_{\ell - 1}(\T_\W)\), we have \(\Nint(\T_{\W(n)}) = \{ n\}\).
Therefore, for all \(n \in \N_{\ell - 1}(\T_\W)\),
\begin{equation*}
0 < \sum_{s \in \Nint(\T_{\W(n)})} \Delta I_Y(s) = \Delta I_Y(n).
\end{equation*}

\vspace{0.3cm}

\noindent
Now assume \(\Delta I_Y(n) > 0\) for all \(n \in \N_{\ell-1}(\T_\W)\).
Then, for any \(n \in \Nint(\T_\W)\), 
\begin{align*}
\sum_{s \in \Nint(\T_{\W(n)})} I_Y(s) &= \sum_{s \in \Nint(\T_{\W(n)}) \setminus \N_{\ell-1}(\T_\W)} I_Y(s) + \sum_{s \in \Nint(\T_{\W(n)}) \cap \N_{\ell-1}(\T_\W)} I_Y(s),\\
&\geq \sum_{s \in \Nint(\T_{\W(n)}) \cap \N_{\ell-1}(\T_\W)} I_Y(s) > 0,
\end{align*}
since \(\Nint(\T_{\W(n)}) \cap \N_{\ell-1}(\T_\W) \neq \emptyset\) for all \(n \in \Nint(\T_\W)\).
Therefore, \(\sum_{s \in \Nint(\T_{\W(n)})} I_Y(s) > 0\) for all \(n \in \Nint(\T_\W)\) and, from Lemma \ref{thm:limitOfQFunction},
\begin{equation*}
\lim_{\beta \to \infty} \tilde Q_Y(n;\beta) = \sum_{s \in \Nint(\T_{\W(n)})} I_Y(s)  > 0, ~~\forall n\in\Nint(\T_\W).
\end{equation*}
Since \(\underset{\beta \to \infty}{\lim}\tilde Q_Y(n;\beta) > 0\) for all \(n \in \Nint(\T_\W)\), all expandable nodes will be expanded by Q-tree search as \(\beta \to \infty\), resulting in the algorithm returning the tree \(\T_\W\).
\end{proof}
%
\section{Proof of Fact \ref{fact:vLeafsRelation}} \label{app:proofFctVleafsRelation}

\begin{proof}
The proof is given by induction.  
First, consider the case when \(n \in r^{-1}(\{0\}) = \mathcal N_{\ell} \left( \mathcal T_{\mathcal W(n)}\right)\).
This leads to 
\begin{equation*}
V^{\lambda}_{\varepsilon}(n) = \frac{1}{2^{d \cdot 0}} \sum_{n'\in \mathcal N_{\text{leaf}}\left( \mathcal T_{\mathcal W(n)} \right)}V^{\lambda}_{\varepsilon}(n') = V^{\lambda}_{\varepsilon}(n),
\end{equation*}
since \(n \in \mathcal N_{\text{leaf}}\left( \mathcal T_{\mathcal W} \right)\).

\vspace{0.3cm}

\noindent
Assume now that the induction hypothesis holds for \(n' \in r^{-1}(\{u\}) = \N_{\ell-u}(\T_\W)\) for some \(u \in \{0,\ldots,\ell-1\}\) and consider any \(n \in r^{-1}(\{u + 1\}) = \mathcal N_{\ell - u - 1}(\T_\W) \).
Then, from the definition of \(V^{\lambda}_{\varepsilon}\), 
\begin{equation*}
V^{\lambda}_{\varepsilon}(n) = \frac{1}{2^d} \sum_{n' \in \mathcal C(n)} V^{\lambda}_{\varepsilon}(n'),
\end{equation*}
and note that \(n' \in \mathcal C(n) \subseteq \mathcal N_{\ell - u}\left( \mathcal T_{\mathcal W} \right)\).
From the induction hypothesis, we have
\begin{equation}\label{eq:FactvLeafsProof_1}
V^{\lambda}_{\varepsilon}(n) = \frac{1}{2^{d(u+1)}} \sum_{n' \in \mathcal C(n)} \sum_{\hat n \in \mathcal N_{\text{leaf}}\left( \mathcal T_{\mathcal W(n')} \right) } V^{\lambda}_{\varepsilon}(\hat n).
\end{equation}
Relation \eqref{eq:FactvLeafsProof_1} is equivalent to
\begin{equation*}
V^{\lambda}_{\varepsilon}(n) = \frac{1}{2^{d(u+1)}}\sum_{n' \in \underset{\hat n \in \mathcal C(n)}{\bigcup}\mathcal N_{\text{leaf}}\left( \mathcal T_{\mathcal W(\hat n)} \right) } V^{\lambda}_{\varepsilon}(n')
= \frac{1}{2^{d r(n)}}\sum_{n' \in \mathcal N_{\text{leaf}}\left( \mathcal T_{\mathcal W(n)} \right) } V^{\lambda}_{\varepsilon}(n'),
\end{equation*}
concluding the proof.
\end{proof}
%
\section{Proof of Proposition \ref{prop:epsFeasibleAbsPath}} \label{app:ProofEpsFeasibleAbsPath}

\begin{proof}
The proof is given by contradiction.
Assume there exists \(\beta > 0\) such that \(\hat J^{\lambda}_{\varepsilon}(\hat \pi;\beta) \geq M_{\varepsilon}^{\lambda}\) for some \(\varepsilon\)-AP \(\hat \pi\).
From this assumption, the definition of \(\hat J^{\lambda}_{\varepsilon}\), and Fact \ref{fact:vLeafsRelation}, we have 
\begin{align*}
M^{\lambda}_{\varepsilon} \leq \sum_{z \in \hat \pi} 2^{d r(z)} V^{\lambda}_{\varepsilon}(z) & = \sum_{z \in  \hat \pi } \sum_{n'\in \mathcal N_{\text{leaf}}\left( \mathcal T_{\mathcal W(z)}\right) }  V^{\lambda}_{\varepsilon}(n'), \\
&= \sum_{n \in \underset{z \in \hat \pi}{\bigcup}\mathcal N_{\text{leaf}}\left( \mathcal T_{\mathcal W(z)}\right) } V^{\lambda}_{\varepsilon}(n).
\end{align*}
Since the path is \(\varepsilon\)-feasible, it follows that \(\underset{z \in \hat \pi}{\bigcup}\mathcal N_{\text{leaf}}\left( \mathcal T_{\mathcal W(z)}\right) \subseteq \mathcal P_{\varepsilon}\) and, along with the definition of \(V^{\lambda}_{\varepsilon}\) as well as the non-negativity of the cost \(c_{\varepsilon}^{\lambda}\), we obtain
\begin{align*}
\sum_{n \in \underset{z \in \hat \pi}{\bigcup} \mathcal N_{\text{leaf}}\left( \mathcal T_{\mathcal W(z)}\right) } V^{\lambda}_{\varepsilon}(n)  \leq \sum_{n \in \mathcal P_{\varepsilon}} c^{\lambda}_{\varepsilon}(n) &= \sum_{x \in \mathcal P_{\varepsilon}} \lambda_1 + \lambda_2 p(y = 1| x),\\
&\leq \sum_{x \in \mathcal P_{\varepsilon}} \lambda_1 + \lambda_2 \varepsilon, \\
& \leq \sum_{x \in \mathcal N_{\text{leaf}}\left( \mathcal T_{\mathcal W}\right) } \lambda_1 + \lambda_2 \varepsilon, \\
& = 2^{d\ell}\left(\lambda_1 + \lambda_2 \varepsilon \right) < 2^{d\ell}\left(\lambda_1 + \lambda_2 \varepsilon \right) + \gamma = M_{\varepsilon}^{\lambda}.
\end{align*}
The above result implies
\begin{equation*}
M_{\varepsilon}^{\lambda} \leq \sum_{z \in \hat \pi } 2^{d r(z)} V^{\lambda}_{\varepsilon}(z) < M_{\varepsilon}^{\lambda},
\end{equation*}
which is a contradiction. 

\vspace{0.3cm}

\noindent
Now assume \(J^{\lambda}_{\varepsilon}(\hat \pi;\beta) < M_{\varepsilon}^{\lambda}\) and define the sets
\begin{align*}
\mathcal A_{\hat \pi} &= \bigcup_{z \in \hat \pi}  \mathcal N_{\text{leaf}}\left( \mathcal T_{\mathcal W(z)}\right) \cap \mathcal P_{\varepsilon}, \\
\mathcal B_{\hat \pi} &=  \bigcup_{z \in \hat \pi}  \mathcal N_{\text{leaf}}\left( \mathcal T_{\mathcal W(z)}\right) \cap \mathcal P^{c}_{\varepsilon}.
\end{align*}
Then, by Fact \ref{fact:vLeafsRelation} we have
\begin{equation*}
\hat J^{\lambda}_{\varepsilon}\left(\hat \pi;\beta \right) = \sum_{z \in \mathcal A_{\hat \pi}} V^{\lambda}_{\varepsilon}(z) + \sum_{z \in \mathcal B_{\hat \pi}}V^{\lambda}_{\varepsilon}(z).
\end{equation*}
From the definition of \(V^{\lambda}_{\varepsilon}\), the above is equivalent to
\begin{equation*}
\hat J^{\lambda}_{\varepsilon}\left(\hat \pi;\beta \right) = \sum_{z \in \mathcal A_{\hat \pi}} V^{\lambda}_{\varepsilon}(z) + \vert \mathcal B_{\hat \pi} \rvert M^{\lambda}_{\varepsilon}.
\end{equation*}
Note that \(\mathcal A_{\hat \pi} \subseteq \mathcal P_{\varepsilon}\) and hence \(0 \leq \sum_{z \in \mathcal A_{\hat \pi}} V^{\lambda}_{\varepsilon}(z) \leq 2^{d\ell}(\lambda_1 + \lambda_2 \varepsilon)< M^{\lambda}_{\varepsilon}\).
Furthermore, observe \(\vert \mathcal B_{\hat \pi} \rvert M^{\lambda}_{\varepsilon} \geq M^{\lambda}_{\varepsilon}\) if \(\vert \mathcal B_{\hat \pi} \rvert \neq 0\) and is \(0\) otherwise.
Thus, if \(\hat J^{\lambda}_{\varepsilon}\left(\hat \pi;\beta \right) < M^{\lambda}_{\varepsilon}\), we have that 
\begin{equation*}
\hat J^{\lambda}_{\varepsilon}\left(\hat \pi;\beta \right) = \underbrace{ \sum_{z \in \mathcal A_{\hat \pi}} V^{\lambda}_{\varepsilon}(z)}_{< M^{\lambda}_{\varepsilon}} + \vert \mathcal B_{\hat \pi} \rvert M_{\varepsilon}^{\lambda} < M^{\lambda}_{\varepsilon},
\end{equation*}
which requires \(\vert \mathcal B_{\hat \pi} \rvert = 0\).
Therefore, if \(\hat J^{\lambda}_{\varepsilon}\left(\hat \pi;\beta \right) < M^{\lambda}_{\varepsilon}\) then \(\mathcal B_{\hat \pi} = \emptyset\).
Hence \(\underset{z \in \hat \pi}{\bigcup}\mathcal N_{\text{leaf}}\left( \mathcal T_{\mathcal W(z)}\right) \subseteq \mathcal P_{\varepsilon}\), thereby implying \(\hat \pi\) is an \(\varepsilon\)-AP.
\end{proof}
%
\section{Proof of Corollary \ref{cor:epsFeasibleFRP}} \label{app:ProofcorEpsFeasibleAbsPath}
\begin{proof}
The proof is identical to the proof of Proposition \ref{prop:epsFeasibleAbsPath}.
\end{proof}

%
\section{Proof of Proposition \ref{prop:vEpsAndObs}} \label{app:proofPropVepsAndObs}

\begin{proof}
Let \(n \in \mathcal N_{\text{int}}\left(\mathcal T_{\mathcal W} \right)\) and assume \(\mathcal N_{\text{leaf}}\left( \mathcal T_{\mathcal W(n)} \right) \cap \mathcal P^{c}_{\varepsilon} \neq \emptyset\).
Define the sets
\begin{align*}
\mathcal A_n &\triangleq \mathcal N_{\text{leaf}}\left( \mathcal T_{\mathcal W(n)}\right) \cap \mathcal P_\varepsilon, \\
\mathcal B_n &\triangleq \mathcal N_{\text{leaf}}\left( \mathcal T_{\mathcal W(n)}\right) \cap \mathcal P^{c}_\varepsilon,
\end{align*}
and note that, by assumption, \(\lvert \mathcal B_n \rvert \neq 0\).
Now, from Fact \ref{fact:vLeafsRelation}, we have that
\begin{align*}
V^{\lambda}_{\varepsilon}(n) &= \frac{1}{2^{d r(n)}}\sum_{n' \in \mathcal N_{\text{leaf}} \left( \mathcal T_{\mathcal W(n)} \right) } V^{\lambda}_{\varepsilon}(n'), \\
&= \frac{1}{2^{d r(n)}} \left[ \sum_{n' \in \mathcal A_n} V^{\lambda}_{\varepsilon}(n') +  \sum_{n' \in \mathcal B_n} V^{\lambda}_{\varepsilon}(n') \right],
\end{align*}
and since the function \(V^{\lambda}_{\varepsilon}\) is non-negative,
\begin{align*}
V^{\lambda}_\varepsilon(n) &= \frac{1}{2^{d r(n)}} \left[ \sum_{n' \in \mathcal A_n} V^{\lambda}_{\varepsilon}(n') +  \sum_{n' \in \mathcal B_n} V^{\lambda}_{\varepsilon}(n') \right], \\
&= \frac{1}{2^{d r(n)}} \left[ \sum_{n' \in \mathcal A_n} V^{\lambda}_{\varepsilon}(n') + \lvert \mathcal B_n \rvert M^{\lambda}_{\varepsilon} \right], \\ 
&\geq \frac{1}{2^{d r(n)}}\lvert \mathcal B_n \rvert M^{\lambda}_{\varepsilon}.
\end{align*}
Now, as \(\lvert \mathcal B_n \rvert \neq 0\), it follows that
\begin{equation*}
V^{\lambda}_\varepsilon(n) \geq \frac{1}{2^{d r(n)}}\lvert \mathcal B_n \rvert M^{\lambda}_{\varepsilon} \geq \frac{1}{2^{d r(n)}}M^{\lambda}_{\varepsilon},
\end{equation*}
and, hence, from the definition of \(M^{\lambda}_{\varepsilon}\), we obtain
\begin{equation} \label{eq:proofPropVeps_1}
V^{\lambda}_\varepsilon(n) \geq \frac{1}{2^{d r(n)}}M^{\lambda}_{\varepsilon} = 2^{d(\ell-r(n))}(\lambda_1 + \lambda_2 \varepsilon) + \frac{1}{2^{d r(n)}}\gamma > 2^{d(\ell-r(n))}(\lambda_1 + \lambda_2 \varepsilon) \geq \lambda_1 + \lambda_2 \varepsilon,
\end{equation}
which follows since \(\nicefrac{1}{2^{d r(n)}} \geq 1\), \(\gamma > 0\) and \(2^{d(\ell-r(n))} \geq 1\).
Relation \eqref{eq:proofPropVeps_1} implies \(V^{\lambda}_\varepsilon(n) > \lambda_1 + \lambda_2 \varepsilon\).

\vspace{0.3cm}

\noindent 
Now let \(n \in \mathcal N_{\text{int}}\left(\mathcal T_{\mathcal W} \right)\) and assume \(V^{\lambda}_\varepsilon(n) > \lambda_1 + \lambda_2 \varepsilon\).
Then, from the definition of \(V^{\lambda}_{\varepsilon}\), we have
\begin{equation*}
V^{\lambda}_{\varepsilon}(n) =  \frac{1}{2^{d r(n)}}\sum_{n' \in \mathcal A_n} V^{\lambda}_{\varepsilon}(n') + \frac{2^{d\ell}(\lambda_1 + \lambda_2\varepsilon) + \gamma}{2^{d r(n)}}\lvert \mathcal B_{n} \rvert,
\end{equation*}
and,
\begin{equation*}
0 \leq \frac{1}{2^{d r(n)}}\sum_{n' \in \mathcal A_n} V^{\lambda}_{\varepsilon}(n')  \leq \frac{1}{2^{d r(n)}}\sum_{n' \in \mathcal A_n} \lambda_1 + \lambda_2 \varepsilon = \frac{1}{2^{d r(n)}}(\lambda_1 + \varepsilon\lambda_2) \lvert \mathcal A_n \rvert \leq \lambda_1 + \varepsilon\lambda_2,
\end{equation*}
where the second inequality follows from \(c_{\varepsilon}^{\lambda}(n') \leq \lambda_1 + \lambda_2 \varepsilon\) for all \(n' \in \mathcal A_{n} \subseteq\mathcal P_{\varepsilon}\) and the third inequality follows from the fact that \(\lvert \mathcal A_n \rvert \leq 2^{d r(n)}\).
Consequently,
\begin{equation*}
0 \leq \frac{1}{2^{d r(n)}}\sum_{n' \in \mathcal A_n} V^{\lambda}_{\varepsilon}(n') \leq \lambda_1 + \varepsilon\lambda_2.
\end{equation*}
Furthermore, note that \(\frac{2^{d\ell}(\lambda_1 + \lambda_2 \varepsilon)+ \gamma}{2^{d r(n)}} > 0\) and 
\begin{equation*}
\frac{2^{d\ell}(\lambda_1 + \lambda_2 \varepsilon)+ \gamma}{2^{d r(n)}}\lvert \mathcal B_{n} \rvert \geq 0.
\end{equation*}
Therefore, \(\left( \nicefrac{(2^{d\ell}(\lambda_1 + \lambda_2\varepsilon) + \gamma)}{2^{d r(n)}} \right)\lvert \mathcal B_{n} \rvert > 0\) only if \(\lvert \mathcal B_n \rvert > 0\), and is \(0\) otherwise.
Then, if \(V^{\lambda}_{\varepsilon}(n) > \lambda_1 + \lambda_2 \varepsilon\) we have that
\begin{equation*}
V^{\lambda}_{\varepsilon}(n) = \underbrace{\frac{1}{2^{d r(n)}}\sum_{n' \in \mathcal A_n} V^{\lambda}_{\varepsilon}(n')}_{\leq \lambda_1 + \lambda_2 \varepsilon} + \underbrace{\frac{2^{d\ell}(\lambda_1 + \lambda_2 \varepsilon) + \gamma}{2^{d r(n)}}\lvert \mathcal B_{n} \rvert}_{\geq 0} > \lambda_1 + \lambda_2 \varepsilon,
\end{equation*}
which requires \(\lvert \mathcal B_n \rvert > 0\). 
Thus, if \(V^{\lambda}_{\varepsilon}(n)>\lambda_1 + \lambda_2 \varepsilon\) then \(\mathcal N_{\text{leaf}}\left( \mathcal T_{\mathcal W(n)} \right) \cap \mathcal P^{c}_{\varepsilon} \neq \emptyset\).
\end{proof}
\end{appendices}
\end{document}